\documentclass{article}
\pdfoutput=1

\newif\ifarxiv
\arxivtrue

\ifarxiv
\else
\usepackage{iclr2023_conference,times}
\fi

\usepackage{microtype}
\usepackage{graphicx}
\usepackage{subfigure}
\usepackage{booktabs} 
\usepackage{amssymb}
\usepackage{bm}
\usepackage{bbm}
\usepackage{amsmath}  
\usepackage{amsthm}
\usepackage{xcolor}
\usepackage{textcomp}
\usepackage{multirow}
\usepackage{mathtools}
\usepackage{algorithm}
\usepackage{algorithmic}
\usepackage{xspace}
\ifarxiv
\usepackage[sort,numbers]{natbib}
\usepackage[margin=1in]{geometry} 
\fi

\usepackage{caption}
\usepackage{tikz}
\usetikzlibrary{arrows}
\usetikzlibrary{positioning}

\tikzset{
  treenode/.style = {align=center, inner sep=0pt, text centered,
    font=\sffamily},
  arn_n/.style = {treenode, circle, black, font=\sffamily\bfseries, draw=black,
    fill=white, text width=1.5em},
  arn_r/.style = {treenode, circle, black, font=\sffamily\bfseries, draw=black,
    fill=white, text width=1.0em},
  arn_x/.style = {treenode, rectangle, draw=black,
    minimum width=0.5em, minimum height=0.5em}
}

\definecolor{tabblue}{HTML}{4e79a7}
\definecolor{tabred}{HTML}{e15759}
\usepackage[hidelinks,colorlinks=true,linkcolor=tabred,citecolor=tabblue]{hyperref}
\usepackage{url}

\usepackage{bm}
\usepackage{enumitem}

\usepackage[capitalise]{cleveref}  

\ifarxiv
\else
\usepackage[compact]{titlesec}
\titlespacing{\section}{0pt}{*0.3}{*0}
\titlespacing{\subsection}{0pt}{*0.15}{*0}

\usepackage[subtle, mathdisplays=tight, charwidths=tight, leading=normal]{savetrees}

\fi

\def\setstretch#1{\renewcommand{\baselinestretch}{#1}}
\newcommand{\softmax}{\mathrm{softmax}}
\newcommand{\Sim}{\mathrm{Sim}}

\newcommand{\vA}{\mathbf{A}}
\newcommand{\vB}{\mathbf{B}}
\newcommand{\vC}{\mathbf{C}}
\newcommand{\vD}{\mathbf{D}}

\newcommand{\vF}{\mathbf{F}}
\newcommand{\vI}{\mathbf{I}}
\newcommand{\vQ}{\mathbf{Q}}
\newcommand{\vK}{\mathbf{K}}
\newcommand{\vV}{\mathbf{V}}

\newcommand{\vP}{\mathbf{P}}

\newcommand{\vW}{\mathbf{W}}

\newcommand{\vO}{\mathbf{O}}

\newcommand{\vM}{\mathbf{M}}

\newtheorem{proposition}{Proposition}
\newtheorem*{remark}{Remark}

\ifdefined\ShowNotes
  \newcommand{\colornote}[3]{{\color{#1}\bf{#2 #3}\normalfont}}
\else
  \newcommand{\colornote}[3]{}
\fi

\definecolor{darkred}{rgb}{0.7,0.1,0.1}
\definecolor{darkgreen}{rgb}{0.1,0.5,0.1}
\definecolor{cyan}{rgb}{0.7,0.0,0.7}
\definecolor{dblue}{rgb}{0.2,0.2,0.8}
\definecolor{maroon}{rgb}{0.76,.13,.28}
\definecolor{burntorange}{rgb}{0.81,.33,0}
\definecolor{royalpurple}{rgb}{0.47,.31,0.66}

\newcommand{\fastfft}{\textsc{FlashConv}\xspace}
\newcommand{\hthree}{\textsc{H3}\xspace}

\ifdefined\ShowNotes
  \newcommand{\num}[1]{{\color{red}\bf{#1}\normalfont}}
\else
  \newcommand{\num}[1]{#1}
\fi

\title{Hungry Hungry Hippos: Towards Language Modeling with State Space Models}

\ifarxiv
  \usepackage{authblk}
  \author[$\dagger$]{Daniel Y. Fu\thanks{Equal Contribution. Order determined by coin flip.}}
  \author[$\dagger$]{Tri Dao$^*$}
  \author[$\ddagger$]{Khaled K. Saab}
  \author[$\dagger\dagger$]{Armin W. Thomas}
  \author[$\ddagger\ddagger$]{\\ Atri Rudra}
  \author[$\dagger$]{Christopher R{\'e}}
  \affil[$\dagger$]{Department of Computer Science, Stanford University}
  \affil[$\ddagger$]{Department of Electrical Engineering, Stanford University}
  \affil[$\dagger\dagger$]{Department of Psychology, Stanford University}
  \affil[$\ddagger\ddagger$]{Department of Computer Science and Engineering, University at Buffalo, SUNY\vspace{4pt}}
  \affil[ ]{{\texttt{\{danfu,tridao\}@cs.stanford.edu}, \texttt{\{ksaab,athms\}@stanford.edu}, \texttt{atri@buffalo.edu}, \texttt{chrismre@cs.stanford.edu}}}

  \date{December 28, 2022}

\else

\author{Daniel Y. Fu\thanks{Equal Contribution. Order determined by coin flip.}~$^{~\dagger}$, Tri Dao$^{* \dagger}$, Khaled K. Saab$^{\dagger}$, Armin W. Thomas$^{\dagger}$, Atri Rudra$^{\ddagger}$, Christopher R\'{e}$^{\dagger}$ \\
$\dagger$ Stanford University, $\ddagger$ University at Buffalo, SUNY \\
\texttt{\{danfu,tridao\}@cs.stanford.edu},~\texttt{\{ksaab,athms\}@stanford.edu}, \\
\texttt{atri@buffalo.edu},~\texttt{chrismre@cs.stanford.edu}
}

\fi

\ifarxiv
\else
\iclrfinalcopy 
\fi
\begin{document}

\maketitle


\begin{abstract}

State space models (SSMs) have demonstrated state-of-the-art sequence modeling performance in some modalities, but underperform attention in language modeling.
Moreover, despite scaling nearly linearly in sequence length instead of quadratically, SSMs are still slower than Transformers due to poor hardware utilization.
In this paper, we make progress on understanding the expressivity gap between SSMs and attention in language modeling, and on reducing the hardware barrier between SSMs and attention.
First, we use synthetic language modeling tasks to understand the gap between SSMs and attention.
We find that existing SSMs struggle with two capabilities: recalling earlier tokens in the sequence and comparing tokens across the sequence.
To understand the impact on language modeling, we propose a new SSM layer, \hthree, that is explicitly designed for these abilities.
\hthree matches attention on the synthetic languages and comes within \num{0.4} PPL of Transformers on OpenWebText.
Furthermore, a hybrid 125M-parameter \hthree-attention model that retains two attention layers
surprisingly outperforms Transformers on OpenWebText by \num{1.0} PPL.
Next, to improve the efficiency of training SSMs on modern hardware,
we propose \fastfft.
\fastfft uses a fused block FFT algorithm to improve efficiency on sequences up to 8K, and introduces a novel state passing algorithm that exploits the recurrent properties of SSMs to scale to longer sequences.
\fastfft yields 2$\times$ speedup on the long-range arena benchmark and allows hybrid language models to generate text \num{2.4$\times$} faster than Transformers.
Using \fastfft, we scale hybrid \hthree-attention language models up to 2.7B parameters on the Pile and find promising initial results, achieving lower perplexity than Transformers and outperforming Transformers in zero- and few-shot learning on a majority of tasks in the SuperGLUE benchmark.

\end{abstract}


\section{Introduction}
\label{sec:intro}

State space models (SSMs) have achieved state-of-the-art sequence modeling performance in domains ranging from time series analysis~\citep{gu2022efficiently} to audio generation~\citep{goel2022s}.
However, they have yet to match the performance of Transformers on language modeling, often underperforming Transformers by multiple points in perplexity~\citep{gu2022efficiently}.
An natural question is whether this gap in performance is due to inherent inductive biases and capabilities in attention~\citep{edelman2022inductive,olsson2022context}, or whether it is a function of the significant organizational resources that have been spent training and tuning large attention-based language models~\citep{chowdhery2022palm,hoffmann2022training,zhang2022opt}, as well as specialized hardware support for attention, ranging from tensor cores~\citep{nvidia2017nvidia} to transformer chips~\citep{nvidia2022nvidia,kao2021optimized}.

We take first steps towards answering these questions in this paper.
First, we use synthetic language modeling tasks to show that there is an expressivity gap between SSMs and attention.
Using our insights, we design a new SSM layer that nearly matches attention in language modeling.
Second, we propose better hardware-aware algorithms for SSMs that allow them to take advantage of modern accelerators---and run faster than attention.

\begin{figure}
    \centering
    \includegraphics[width=\textwidth]{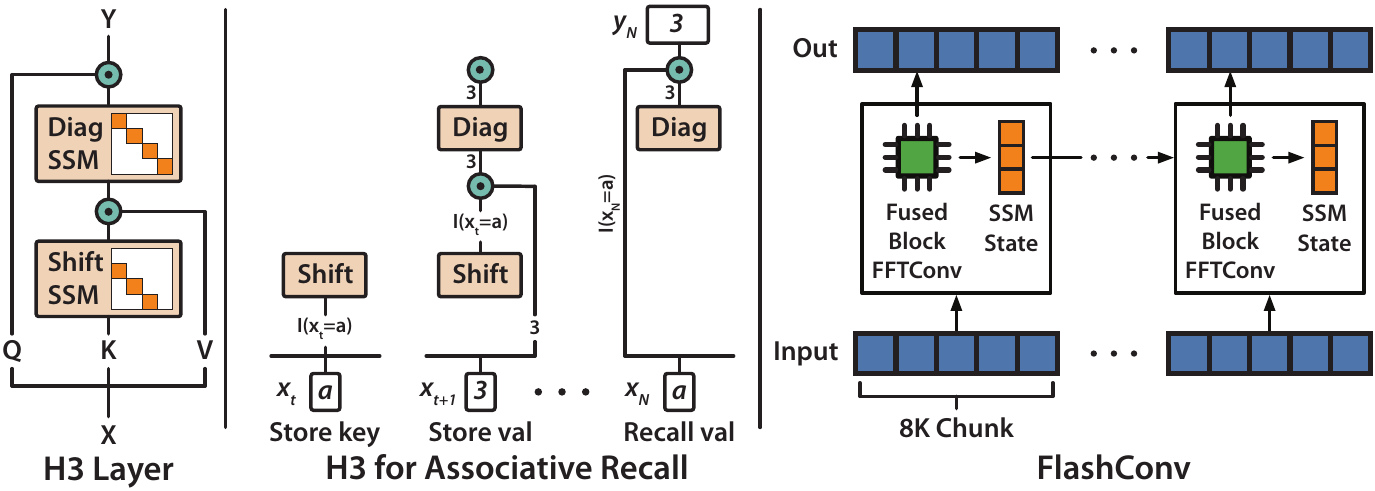}
    \caption{\label{fig:banner}
    Left: \hthree stacks two discrete SSMs with shift and diagonal matrices and uses multiplicative interactions between input projections and their outputs to model comparisons between points in a sequence.
    Middle: \hthree can perform associative recall---which is easy for attention, but not existing SSMs.
    Right: \fastfft uses a new state-passing algorithm over fused block FFTConv to increase hardware efficiency of SSMs, allowing \hthree to scale to billion-parameter models.
    }
    \vspace{-1.25em}
\end{figure}

\textbf{Understanding the Expressivity Gap.}
To understand the gap between SSMs and attention, we draw on synthetic language modeling tasks that have been proposed as a mechanistic basis for in-context learning in Transformers~\citep{olsson2022context}
These synthetic languages focus on the ability to manipulate text---recalling tokens from earlier time steps, or comparing tokens from different points in a sequence.
We find that existing SSMs struggle to model these synthetic languages.
To probe how important these skills are for language modeling, we propose \hthree (Hungry Hungry Hippo), a new SSM-based layer designed to solve these language modeling tasks.
\hthree stacks two SSMs, with multiplicative interactions between their outputs and input projections.
The SSMs allow \hthree to keep a log of tokens (to recall them later), while the multiplicative interactions allow for comparisons across the sequence.

\hthree matches attention on the synthetic languages and almost closes the gap with Transformers on language modeling---coming within \num{0.4} perplexity of Transformers on OpenWebText (compared to \num{3.4} ppl for existing SSMs---even those explicitly designed for language modeling~\citep{mehta2022long}).
Furthermore, a simple hybrid \hthree-attention model that retains two attention layers surprisingly \textit{outperforms} Transformers on OpenWebText by \num{1.0} perplexity.
To further evaluate \hthree on language modeling, we train 125M-, 355M-, 1.3B-, and 2.7B-parameter hybrid \hthree-attention language models on the Pile~\citep{gao2020pile}, using hyperparameters from GPT-3~\citep{brown2020language}.
These hybrid models outperform Transformer-based language models of the same size in perplexity, and match or outperform them on a majority of tasks in the SuperGLUE benchmark in zero- and few-shot learning.
Since the SSM layers in these hybrid models admit a recurrent view, they can also perform \num{2.4$\times$} faster inference than Transformers.

\textbf{Scaling SSMs.}
Next, we improve the efficiency of SSMs on modern hardware, to reduce the hardware barrier between attention and SSMs.
SSMs scale nearly linearly in sequence length instead of quadratically like attention, but still run slower on modern hardware due to poor hardware utilization.
To close this gap, we propose \fastfft, a hierarchical algorithm for computing SSMs, inspired by IO-Aware attention~\citep{dao2022flashattention}.
The technical challenge is that SSMs require a FFT-based convolution over the input sequence, which requires an FFT, pointwise multiply, and inverse FFT.
When implemented in cuFFT~\citep{cufft}, this operation incurs expensive GPU memory reads/writes, and cannot utilize the specialized matrix multiply units available on modern hardware\footnote{An A100 GPU has a maximum of 312 TFLOPs/s of FP16 with
tensor cores, but only 20 TFLOPs/s of FP32 (and 40 TFLOPs/s of FP16) without
tensor cores~\citep{nvidia2020nvidia}. This trend started with the V100 GPUs~\citep{nvidia2017nvidia} and has continued with the
H100 GPUs~\citep{nvidia2022nvidia}.}.
To use specialized matrix multiply units, we appeal to classical techniques that split the FFT into blocks and compute it using a series of matrix multiplications.
Combined with kernel fusion, this ``block'' FFT solution increases hardware efficiency, but only as long as the sequence length can fit into GPU SRAM (on-chip memory, analogous to L1 cache on the CPU)---up to sequence length 8K on modern A100.

To scale to sequences longer than 8K, we propose a \textit{state passing} algorithm (Figure~\ref{fig:banner} right), specialized to SSMs.
The key insight is that we can use the recurrent properties of SSMs to process the input in chunks---as long as we keep track of an additional state vector.
The state passing algorithm splits the input into the largest chunks that can fit into GPU SRAM, efficiently computes the FFT-based convolution using block FFT, and updates an intermediate state to start the next chunk.
Using this state-passing algorithm, \fastfft can scale SSMs to \textit{any} sequence length---even longer than can fit on GPU SRAM at once---while maintaining a \textit{near linear} compute complexity.
\fastfft sets state-of-the-art speed on long range arena using S4~\citep{gu2022efficiently}, outperforming Transformers by \num{5.8$\times$} and previous S4 models by \num{2$\times$}.
\fastfft trains \hthree\ \num{4-8$\times$} times faster than attention for long sequences, and is a critical component for scaling to billion-parameter models\footnote{Code for H3 is available at \url{https://github.com/HazyResearch/H3} }.


\section{Background}
\label{sec:background}

We present some background on state space models and linear attention, which inspired our H3 layer.

\subsection{State Space Models}

A continuous-time state-space representation~\citep{brogan1974modern} defines a linear mapping from an
input signal $u(t) \in \mathbb{R}$ (as a function of time $t$) to an output signal $y(t) \in \mathbb{R}$ through a state-variable
$x(t) \in \mathbb{R}^m$, with the following differential equation, for some matrices $\vA \in \mathbb{R}^{m \times m}$, $\vB \in \mathbb{R}^{m \times 1}$, $\vC \in \mathbb{R}^{1 \times m}$,
$\vD \in \mathbb{R}^{1 \times 1}$: $\dot{x}(t) = \vA x(t) + \vB u(t)$, $y(t) = \vC x(t) + \vD u(t)$.

Similarly, a discrete-time state-space representation defines a linear mapping
from a discrete input signal $u_i$ (for $i = 1, 2, \dots$) to a discrete output signal
$y_i$ though a state-variable $x_i \in \mathbb{R}^m$:
\begin{align*}
  x_i &= \vA x_{i-1} + \vB u_i \\
  y_i &= \vC x_i + \vD u_i.
\end{align*}

A state-space model (SSM) uses these representations as a layer in a deep learning
pipeline, where the matrices $\vA, \vB, \vC, \vD$ are learned from data (e.g.,
with gradient-based optimization).
One often has $d$ of these SSMs in parallel, each corresponding to one hidden
dimension.
To preserve the sequence history, HiPPO~\citep{gu2020hippo} projects the history
on a basis of orthogonal polynomials, which translates to having SSMs whose
$\vA, \vB$ matrices are initialized to some special matrices.

This recurrent form of SSMs allows efficient inference (i.e., generation): to
generate the output of the next time-step, one only needs the state of the
current time-step, not the entire input history.
Furthermore, SSMs can freely extrapolate to sequences longer than seen during training.

\textbf{SSMs as Convolution.}
For efficient training, given the
entire sequence of the input $u_1, \dots, u_N$, the output sequence
$y_1, \dots, y_N$ can also be written as the convolution of the input with the
filter~\citep{gu2021combining}:
\begin{equation*}
  f = [\vC \vB, \vC\vA\vB, \vC\vA^2\vB, \dots, \vC\vA^{N-1}\vB].
\end{equation*}
That is, from an initial condition $x_0$, we have
$y_i = \vC\vA^i\vB x_0 + (f \ast u)_i + \vD u_i$, where $(f \ast u)$ denotes a
linear convolution between $f$ and $u$.
If we set the initial condition $x_0$ to be zero, then $y$ is exactly a
linear convolution of $u$, with a residual connection $\vD u$.
More generally, any linear time-invariant system (of which SSMs are a special case) can
be written as a convolution.

Given a 1D input sequence $u \in \mathbb{R}^{N}$ of length $N$, we denote the 1D output
sequence $y \in \mathbb{R}^N$ of an SSM parameterized by matrices $\vA, \vB, \vC, \vD$ as
\begin{equation*}
  y = \mathrm{SSM}_{\vA, \vB, \vC, \vD}(u).
\end{equation*}
To simplify notation, we omit the reference to $\vA, \vB, \vC, \vD$ and write
$y = \mathrm{SSM}(u)$ if they are clear from context.
When $u$ is multidimensional of dimension $d$, we stack $d$ of these SSMs
together that defines a mapping from $u \in \mathbb{R}^{N \times d}$ to $y \in \mathbb{R}^{N \times d}$, using
the same notation $y = \mathrm{SSM}(u)$.

To construct the filter $f$ from $\vA, \vB, \vC$ efficiently, $\vA$ is often constrained to
be diagonal~\citep{gupta2022diagonal,gu2022parameterization}, or diagonal plus
low-rank~\citep{gu2022efficiently}.

\textbf{SSM through FFTs.}
Computing the convolution naively through conventional matrix operations is expensive
for long kernels, scaling as $O(N^2)$.
Instead, we can use FFTs: take the FFT of $f$ and $u$, multiply them together pointwise, and then take the inverse FFT.
This yields an $O(N \log N)$ algorithm.

\subsection{Linear attention}

We describe linear attention~\citep{katharopoulos2020transformers} and its connection to RNNs, which inspired our model design~(\cref{sec:method}).

In standard attention~\citep{vaswani2017attention}, we have $N$ query/key/value tokens $Q_i, K_i, V_i \in \mathbb{R}^d$ for
$i = 1, \dots, N$, where $N$ is the sequence length and $d$ is the head dimension.
For some similarity metric $\Sim \colon \mathbb{R}^d \times \mathbb{R}^d \to \mathbb{R}$, we want to compute the output:
\begin{equation*}
  O_i = \frac{\sum_{j=1}^i\Sim(Q_i, K_j) V_j}{\sum_{j=1}^i \Sim(Q_i, K_j)} \in \mathbb{R}^d.
\end{equation*}
For standard softmax attention, $\Sim(q, k) = e^{q^\top k}$ (often the dot
product is scaled by $1/\sqrt{d}$).
Linear attention makes the assumption that $\Sim$ has the form
$\Sim(q, k) = \phi(q)^\top \phi(k),$
for some (nonlinear) function $\phi$.
The output is then $O_i = \frac{\phi(Q_i)^\top \sum_{j=1}^{i} \phi(K_j) V_j^\top }{\phi(Q_i)^\top \sum_{j=1}^{i} \phi(K_j)}$.
Let
$S_i = \sum_{j=1}^{i} \phi(K_j) V_j^\top \in \mathbb{R}^{d \times d}$, $z_i = \sum_{j=1}^{i} \phi(K_j) \in \mathbb{R}^d$, $d_i = \phi(Q_{i})^\top z_i \in \mathbb{R}$.
Then $O_i = \frac{\phi(Q_i)^\top S_i}{d_i}$.
This connects linear attention to RNNs: the output $O_i$ is a function of $S_i$
and $z_i$, both of which are incrementally updated (as cumulative sums).



\section{Hungry Hungry Hippos Layer to Model Discrete Sequences}
\label{sec:method}

To understand the gap between SSMs and attention on language modeling, we
examine two synthetic language modeling tasks.
These tasks motivate our \hthree layer to add a discrete SSM (based on shift matrix) and multiplicative interaction to effectively model discrete sequences.
We then show that the \hthree layer is expressive enough to solve these synthetic tasks, and that this understanding leads to better performance on a real language modeling benchmark.

\subsection{Motivation: Synthetic Language Modeling Tasks\label{sec:synthetics}}

We describe two closely-related synthetic tasks, summarized in Table~\ref{table:synthetic_tasks}. 
Olsson et al.~\citep{olsson2022context} argue that the ability to solve (variants of) these tasks accounts for the majority of the in-context learning capability of Transformers, and more intuition is given in Appendix~\ref{sec:app_exp_details}. 
\begin{table}[h]
    \small
    \centering
    \caption{\label{table:synthetic_tasks} Synthetic language modeling tasks.}
    {
        \begin{tabular}{@{}|l|l|l|c|c|@{}}
        \hline
        Task & Input & Output & Sequence Length & Vocab Size  \\ 
        \hline
        Induction Head & \textit{a b c d e $\vdash$ f g h i $\dots$ x y z $\vdash$} & \textit{f} & 30 & 20 \\
        Associative Recall & \textit{a 2 c 4 b 3 d 1 a} & \textit{2} & 20 & 10 \\ \hline
        \end{tabular}
    }
\end{table}

The \textbf{Induction Head} task tests how well a model can recall content after a special token (e.g., $\vdash$ in~\cref{table:synthetic_tasks}).
At the end of the sequence, the model must recall the token that appeared immediately after the special token earlier in the sequence. \textbf{Associative Recall}~\citep{ba2016using} is similar to the induction head task, but requires the model to remember multiple key-value pairs.
At the end of the sequence, the model must recall a specific value belonging to a specific key.
\begin{table}[h]
    \small
    \centering
    \caption{\label{table:synthetics} Evaluation of 2-layer models on synthetic language tasks.}
    {
        \begin{tabular}{@{}|c|c|ccc|c|@{}}
        \hline
        Task & Random & S4D & Gated State Spaces & H3 & Attention  \\ 
        \hline
        Induction Head & 5.0 & 35.6 & 6.8 & \textbf{100.0} & \textbf{100.0} \\
        Associative Recall & 25.0 & 86.0 & 78.0 & 99.8 & \textbf{100.0}  \\ \hline
        \end{tabular}
    }
\end{table}

Table~\ref{table:synthetics} (for two-layer models) shows that S4D~\citep{gu2022parameterization} and Gated State Spaces~\citep{mehta2022long} both fail to model these synthetic languages, which suggests they may not have the expressivity for general language.
We argue that these failures suggest two missing capabilities: (i) to remember tokens that appear after a particular event (e.g., the special token in the induction head task), and (ii) to compare tokens across the sequence (e.g., comparing keys to decide which value to recall).
Attention has both these capabilities: it can compare tokens by constructing the \textit{quadratic} attention matrix $\vQ \vK^\top$, and it can recall tokens by direct copying (multiplying $\softmax(\vQ \vK^\top)$ with $\vV$).
In Section~\ref{sec:method_h3}, we design our new layer \hthree to enable these capabilities in SSMs, narrowing the expressivity gap between SSMs and attention.

\subsection{\hthree Layer\label{sec:method_h3}}
\hthree uses SSMs with shift and diagonal matrices, along with multiplicative
operations against projections of the input to capture the missing capabilities identified by the synthetics.

\textbf{High-level Intuition.}
(i) To remember tokens from the past, we want the state $x_i$ to copy from the input $u_i$, and then pass that information to the next state $x_{i+1}$. As $x_{i+1}$ relates to $x_i$by $\vA x_i$, we use a discrete SSM with a shift matrix $\vA$ (described formally below) that shifts the elements of a state vector (e.g., mapping $[a, b, c] \to [0, a, b]$).
(ii) To compare tokens across the sequence, we use multiplicative interaction: the output of an SSM, containing information from previous time steps, is multiplied with the input at the current time steps, thus measuring similarity between tokens.

\hthree is loosely inspired by linear attention (\cref{sec:background}): we project the input $u$ to get three signals $\vQ, \vK, \vV$.
Then we replace the non-linearity $\phi(\vK)$ with an SSM where $\vA$ is a shift matrix ($\mathrm{SSM}_\mathrm{shift}$), and we replace the summation $S_i$ with a SSM with diagonal $\vA$ ($\mathrm{SSM}_\mathrm{diag}$).
The output, for the case of head dimension $d_h = 1$, is:
\begin{equation*}
  \vQ \odot \mathrm{SSM}_\mathrm{diag}(\mathrm{SSM}_\mathrm{shift}(\vK) \odot \vV),
\end{equation*}
where $\odot$ denotes pointwise multiplication.
We can view this form as stacking two SSMs with multiplicative
interaction (each is a ``hungry hippo'', hence the name of our layer).
A more formal connection between linear attention, time-varying systems, and \hthree can be found in Appendix~\ref{app:linear_attention}.

\textbf{Remembering Key Tokens: Shift and Diagonal SSMs.}
The shift and diagonal SSMs are designed to address the capability to log tokens after particular events.
In the shift SSM, we constrain $\vA \in \mathbb{R}^{m \times m}$ to be a shift matrix
$
\vA_{i, j} = 
    \begin{cases}
    1 & \text{for } i - 1 = j\\
    0 & \text{otherwise}
    \end{cases}
$.
The action of this matrix on the hidden state $x_i$ is to shift each coordinate
down by one---thereby creating a ``memory'' of the previous states.
For example, if $\vB = e_1$, the first basis vector, then
$x_i = [u_i, u_{i-1}, \dots, u_{i-m+1}]$ contains the inputs from the previous $m$
time steps.
We learn $\vB$ and $\vC$ ($\vB$ can also be fixed to $e_1$ for
simplicity, in which case the output is a 1D conv.\ with kernel size $m$).

The diagonal SSM constrains $\vA$ to be diagonal and initializes it from the diagonal version of HiPPO (S4D~\citep{gu2022parameterization}).
This parameterization allows the model to remember state over the entire sequence.
The shift SSM can detect when a particular event occurs, and the diagonal SSM can remember a token afterwards for the rest of the sequence.

\textbf{Multiplicative Interaction for Comparison.}
We take the multiplicative interactions from linear attention, but they provide another missing capability when combined with a shift matrix: comparing tokens across the sequence.
The multiplicative interactions between the output of the shift SSM and the $\vV$ projection mimics local multiplicative interactions in linear attention (depending on the size of the hidden state).
Similarly, multiplicative interactions with the $\vQ$ projection and the output of the diagonal SSM allows comparisons between tokens over the entire sequence.

\textbf{\hthree Layer.} The overall layer is given in Algorithm~\ref{alg:h3} and shown schematically in Figure~\ref{fig:banner} (left).
We use the \hthree layer to construct a model in the same style as Transformers by interleaving it with MLPs,
connected by residual connection and layer norm (i.e., pre-norm architecture~\citep{baevski2018adaptive}).
We will also consider a hybrid \hthree-attention model (two attention layers while the rest are \hthree, \cref{sec:expressivity,sec:evaluation}).
\begin{algorithm}[H]
  \algsetup{linenosize=\tiny}
  \caption{\label{alg:h3} H3 Layer}
  \small
  \begin{algorithmic}[1]
    \REQUIRE Input sequence $u \in \mathbb{R}^{N \times d}$ from the previous layer, weight
    matrices $\vW_Q, \vW_K, \vW_V, \vW_O \in \mathbb{R}^{d \times d}$, a shift SSM $\mathrm{SSM}_\mathrm{shift}$, a diagonal SSM $\mathrm{SSM}_{\mathrm{diag}}$, head dimension $d_h$.
    \STATE Compute $\vQ = u \vW_Q, \vK = u \vW_K, \vV = u \vW_V \in \mathbb{R}^{N \times d}$.
    \STATE Pass $\vK$ through the shift SSM: $\overline{\vK} = \mathrm{SSM}_\mathrm{shift}(\vK) \in \mathbb{R}^{N \times d}$.
    \STATE Split $\vQ, \overline{\vK}, \vV$ into $H$ ``heads'' ($\vQ^{(h)}, {\overline{\vK}}^{(h)}, \vV^{(h)}$
    for $h = 1, \dots, H$), each a sequence of $N$ vectors of size $d_{h} = d / H$.
    \FOR{$1 \leq h \leq H$}
    \STATE Take the batched outer product ${\overline{\vK}}^{(h)} (\vV^{(h)})^\top \in \mathbb{R}^{N \times d_h \times d_h}$ (batched in the $N$-dimension) and pass it through a diagonal SSM: $\vK\vV^{(h)} = \mathrm{SSM}_{\mathrm{diag}}({\overline{\vK}}^{(h)} (\vV^{(h)})^\top) \in \mathbb{R}^{N \times d_{h} \times {d_{h}}}$.
    \STATE Batch-multiply by $\vQ$:  $\vO^{(h)} = [\vQ^{(h)}_1 \vK\vV^{(h)}_1, \dots, \vQ^{(h)}_N \vK\vV^{(h)}_N]\in  \mathbb{R}^{N \times d_h}$ (batched in the $N$-dimension).
    \ENDFOR
    \STATE Concatenate the output $\vO^{(h)}$ of each head, and multiply by the output
    projection matrix $\vW_O \in \mathbb{R}^{d \times d}$.
  \end{algorithmic}
\end{algorithm}

\paragraph{Efficiency}
We show that \hthree scales as $O(N \log N)$ with
sequence length $N$---asymptotically more efficient than attention, which typically requires $O(N^2d)$ time and $O(N^2)$ space\footnote{There are several memory-efficient
algorithms for attention~\citep{rabe2021self,dao2022flashattention}, though
their time complexity is still quadratic in $N$, which is a lower-bound for attention~\citep{keles2022computational}.} (proof in~\cref{sec:app_h3_complexity}).
\begin{proposition}\label{thm:h3_complexity}
  Let $N$ be the sequence length, $d$ be the hidden dimension, and assume that
  the head dimension $d_h$ is of order $O(1)$.
  Then the H3 layer takes $O(d^2N + d N \log N)$ time and $O(dN)$ space to compute.
\end{proposition}

\subsection{Expressivity}
\label{sec:expressivity}

We show that \hthree can model our synthetic languages, as well as natural language on OpenWebText~\citep{Gokaslan2019OpenWeb}.
We also present a hybrid \hthree-attention extension that outperforms Transformers on OpenWebText.

\textbf{Mechanism for Solving Associative Recall with H3.}
\hthree is expressive enough to solve our synthetic language modeling tasks, as shown in Table~\ref{table:synthetics}.
Figure~\ref{fig:banner} (middle) shows a mechanism for a single \hthree layer to solve the associative recall task for a particular key-value pair $(a, 3)$.
The shift SSM and following multiplicative interaction act as a gate on whether to let a value through to the diagonal SSM, based on whether the previous token was key $a$.
The diagonal SSM stores the value $3$ in memory, and continually outputs it.
The final multiplicative interaction gates whether to let the diagonal SSM's output through---based on whether the current input token is the key $a$.
We formally construct the weights of an \hthree layer to solve this task in Appendix~\ref{sec:app_expressivity}.

\begin{table}[h]
    \small
    \centering
    \caption{\label{table:ablations} Perplexity of SSM variants compared to
      Transformers on OpenWebText. All models have 12 layers, with size around 125M, and are trained
      with the same hyperpameters, for 50B tokens.}
    {
        \begin{tabular}{@{}|ccccc|c|@{}}
            \hline
        \hthree & \hthree Hybrid (2 Attn) & S4D & GSS & GSS Hybrid (2 Attn) & Transformer  \\ 
        \hline
        21.0 & \textbf{19.6} & 24.9 & 24.0 & 19.8 & 20.6 \\ \hline
        \end{tabular}
    }
\end{table}

\textbf{Better Synthetic Language Modeling Translates to Better Natural Language Modeling.}
We validate that when H3 can solve these synthetic tasks, it also improves the modeling capability on natural language (e.g., on the OpenWebText dataset).
As shown in Table~\ref{table:ablations}, \hthree comes within 0.4 perplexity points of Transformers when trained for 50B tokens on OpenWebText, and performs much better than existing SSM variants (S4D, GSS), by $3-3.9$ points.

\textbf{Extension: H3-attention Hybrid Model.}
A simple hybrid \hthree-attention language model surprisingly outperforms Transformers on OpenWebText by 1.0 point.
Our hybrid model simply retains two self-attention layers: one in the second layer, and one in the middle (layer $2 + N / 2$ for an $N$-layer model, $N$ even).
The \hthree-attention hybrid also outperforms the GSS-attention hybrid~\citep{mehta2022long}.

\section{\fastfft: Efficiently Training SSMs\label{sec:efficiency}}

To improve the efficiency of SSMs on modern hardware, we propose \fastfft.
\fastfft fuses the FFT, pointwise multiply, and inverse FFT to reduce memory
reads/writes.
It also uses a block FFT algorithm to make use of specialized matrix multiply units (e.g., tensor cores on A100) for sequence lengths up to 8K.
For sequences longer than 8K, the computation no longer fits in GPU SRAM\footnote{SRAM, or on-chip memory, is much faster than off-chip GPU memory, but usually much smaller, on the order of around 100KB for each streaming processor.}, so we
propose a novel state-passing algorithm that splits the sequence into chunks to
compute the FFT convolution one chunk at a time.
\fastfft can speed up any SSMs (not just \hthree).

\subsection{Fused Block FFTConv}
\label{sec:systolic_fft}

We deploy two techniques to speed up the FFT-based convolution for sequences shorter than 8K: kernel fusion and block FFT.
Kernel fusion addresses IO bottlenecks due to reading and writing of intermediate results, while block FFT allows the FFT-based convolution to utilize specialized matrix multiplication units.
These techniques allow us to speed up FFTConv
by 2$\times$ (\cref{sec:eval_efficiency}) for sequences shorter than 8k.

\textbf{Kernel Fusion.}
Naive implementations of FFTConv using standard libraries such as cuFFT are IO-bound due to repeated reading and writing of intermediate results.
The FFT convolution in an SSM with input $u$ and filter $f$ has the form $iFFT(FFT(u) \odot FFT(f))$ (where $\odot$ denotes pointwise multiplication).
It requires reading and writing intermediate results to GPU memory---which can dominate the runtime.
Following \textsc{FlashAttention}~\citep{dao2022flashattention}, we first fuse the entire FFTConv into a single kernel and compute it in SRAM to avoid this overhead.

\textbf{Block FFT.}
To further speed up the computation of FFT-based convolution, we exploit specialized matrix multiplication hardware on modern GPUs
(e.g., Tensor Cores on Nvidia GPUs perform fast $16 \times 16$ matrix multiplication).
We appeal to classical results that show that the FFT can be written as a series of block-diagonal matrix
multiplications interleaved with permutation.
We note that such algorithms are not new, but our setting (fused FFTConv on GPU) introduces new bottlenecks---by removing the IO bottlenecks, compute becomes the bottleneck (note that a single FFT on GPU is usually IO bound).

Suppose that we want to perform an $N$-point FFT, which is equivalent to
multiply by the DFT matrix $\vF_N$.
Suppose that $N = N_1 N_2$ for some integers $N_1, N_2$.
By the Cooley-Tukey decomposition of the DFT~\citep{cooley1965an,
  bailey1990ffts} (also known as the four-step FFT algorithm),
we can write $\vF_N = \vP (\vI_{N_2} \otimes \vF_{N_1}) \vP^\top \vD (\vI_{N_1} \otimes \vF_{N_2}) \vP$,
where $\vP$ denotes a fixed permutation that reshapes the input as a $N_1 \times N_2$
array and then transpose it, $\otimes$ denotes Kroneker product, $\vD$ is a $N \times N$
diagonal matrix (called the twiddle factors)~\citep{dao2022monarch}, and $\vI_{N_i}$ and $\vF_{N_i}$ are the identity and DFT matrix of size $N_i \times N_i$.
As $\vI_{N_2} \otimes \vF_{N_1}$ and $\vI_{N_1} \otimes \vF_{N_2}$ are just block-diagonal matrices,
we can make use of specialized matmul units to perform these multiplications.
Similarly, if $N = N_1 N_2 N_3$ then we can decompose the $N$-point FFT into a
series of (block) FFT of size $N_1$, $N_2$, and $N_3$, interleaved by
permutation.

The block FFT algorithm incurs $O(N r \log N / \log r)$ FLOPs for a sequence length $N$, if $N$ can be written as $r^p$ for two integers $r, p$.
This incurs more FLOPs than standard FFT $(O(N\log N))$, but can run faster when
we using specialized matrix multiplication hardware.

\subsection{State-Passing}
However, the fused kernel cannot run if the sequence is too long to fit into GPU SRAM (longer than 8K on A100).
We show how to exploit the particular form of the FFT in SSM to speed it up for long
sequences.

The recurrent nature of SSMs allows us to split the FFTConv of a length-$N$ sequence into chunks of size $N'$ each ($N'$ is the longest FFT we can fit into SRAM), assuming $N$ is a multiple of $N'$).
We use FFTConv to compute each chunk and use a recurrence to connect the chunks.
In particular, we split the inputs $u$ into $C = N/N'$ chunks $u^{(c)} \in \mathbb{R}^{N'}$ for $c=1, \dots, C$.
Similarly, split the states $x$ into $x^{(c)} \in \mathbb{R}^{N' \times m}$ and the output $y$ into $y^{(c)} \in \mathbb{R}^{N'}$ for $i = 1, \dots, C$.
We will only need the end-state $x_{N'}^{(c)}$ of each chunk $c$.

Let $f = [\vC \vB, \vC\vA\vB, \vC\vA^2\vB, \dots, \vC\vA^{N'-1}\vB]$ be the SSM filter.
Recall from \cref{sec:background} that for each chunk $c$, $y_i^{(c)} = \vC\vA^i\vB x_{N'}^{(c-1)} + (f \ast u^{(c)})_i + \vD u_i^{(c)}$, since $x_{N'}^{(c-1)}$, the end-state of the previous chunk $(c-1)$ is the initial condition for the current chunk $c$.
In vector notation, $y^{(c)} = \vM_{xy} x_{N'}^{(c-1)} + f \ast u^{(c)} + \vD u^{(c)}$ for some matrix $\vM_{xy} \in \mathbb{R}^{N' \times m}$.
Additionally we need to update the end-state of each chunk with $x_{N'}^{c} = \vA^{N'} x_{N'}^{(c-1)} + \vM_{ux} u^{(c)}$ for some matrix $\vM_{ux}^{m \times N'}$ (derivation in Appendix~\ref{sec:state-passing-matrices}).
In essence, we can compute the output for each chunk with FFT-based convolution as long as we remember the end-state of the previous chunk, and the end-state of each chunk can be updated recurrently.
This yields a state-passing algorithm for long sequences, where we only compute FFT of length $N'$, and update some hidden state each iteration.

Let \textsc{BlockFFTConv} refer to our fused block FFTConv kernel.
Then, the state-passing algorithm for 1D input is given by Algorithm~\ref{alg:statepassing}.
For inputs of dimension $d$ where we stack $d$ SSMs, we simply batch~\cref{alg:statepassing} along the $d$-dimension.
\begin{algorithm}[h]
  \algsetup{linenosize=\tiny}
  \caption{\label{alg:statepassing} State Passing Algorithm}
  \small
  \begin{algorithmic}[1]
    \REQUIRE Input $u \in \mathbb{R}^{N}$, SSM parameterized by matrices $\vA \in \mathbb{R}^{m \times m}$, $\vB \in \mathbb{R}^{m \times 1}$, $\vC \in \mathbb{R}^{1 \times m}$, $\vD \in \mathbb{R}^{1 \times 1}$, chunk size $N'$ where $N$ is a multiple of $N'$.
    \STATE Precompute $\vA^{N'} \in \mathbb{R}^{m \times m}$, $\vM_{ux} = [\vA^{N'-1}\vB, \dots, \vB] \in \mathbb{R}^{m \times N'}$, $\vM_{xy} = [\vC, \vC\vA, \dots, \vC\vA^{N'-1}] \in \mathbb{R}^{N' \times m}$.
    \STATE Split the inputs $u_{1:N}$ into $C = N/N'$ chunks $u_{1:N'}^{(c)}$ for $c=1, \dots, C$.
    \STATE Let the initial state be $x_{N'}^{(0)} = 0 \in \mathbb{R}^m$.
    \FOR{$1 \leq c \leq C$}
      \STATE Compute $y^{(c)} = \vM_{xy} x_{N'}^{(c-1)} + $ \textsc{BlockFFTConv}($f$, $u_j$) $+ \vD u^{(c)} \in \mathbb{R}^{N'}$.
      \STATE Update state: $x_{N'}^{(c)} = \vA^{N'} x_{N'}^{(c-1)} + \vM_{ux} u^{(c)}$.
    \ENDFOR
    \STATE Return $y = [y^{(1)} \dots y^{(C)}]$.
  \end{algorithmic}
\end{algorithm}

We prove that~\cref{alg:statepassing} yields the same output as if one has computed the SSM using a large FFT of size $N$ (proof in~\cref{sec:app_statepassing_correctness_proof}):
\begin{proposition}\label{thm:statepassing_correctness}
  For input $u \in \mathbb{R}^N$ and matrices $\vA, \vB, \vC, \vD$, the output $y \in \mathbb{R}^N$ returned by~\cref{alg:statepassing} is the same as the output defined by the SSM parameterized by $\vA, \vB, \vC, \vD$.
\end{proposition}


\section{\hthree Evaluation}
\label{sec:evaluation}

\begin{table}[t]
    \small
    \centering
    \caption{\label{table:gpt} Perplexity (lower is better) of models on the Pile, OpenWebText and
      WikiText-103. GPT-Neo and hybrid \hthree are trained on the Pile, while GPT2 is
      trained on WebText. All models use the same GPT2 tokenizer. We report the
      perplexity of GPT-2 models on the Pile ($^*$) for context, though the performance is not directly comparable since they were trained on different data.}
    {
      \begin{tabular}{@{}|c|c|cc|@{}}
      \hline
        Model & Pile & OpenWebText & WikiText103 \\ 
        \hline
        GPT-2 small (125M) & 19.0* & 22.6 & 29.9 \\
        GPT-Neo-125M & 9.4 & 22.6 & 26.3 \\
        \textbf{Hybrid H3-125M} & \textbf{8.8} & \textbf{20.9} & \textbf{23.7} \\ \hline 
        GPT-2 medium (355M) & 13.9* & 17.0 & 21.8 \\ 
        \textbf{Hybrid H3-355M} & \textbf{7.1} & \textbf{15.9} & \textbf{16.9} \\ \hline
        GPT-2 XL (1.5B) & 12.4* & 12.9 & 17.0 \\
        GPT-Neo-1.3B & 6.2 & 13.1 & 13.3 \\
        \textbf{Hybrid H3-1.3B} & \textbf{6.0} & \textbf{12.4} & \textbf{12.5} \\
        \hline
        GPT-Neo-2.7B & 5.7 & 11.7 & 11.5 \\
        \textbf{Hybrid H3-2.7B} & \textbf{5.4} & \textbf{11.0} & \textbf{10.6} \\
        \hline
      \end{tabular}
    }
  \end{table}

\begin{table}[t]
    \scriptsize
    \centering
    \caption{\label{table:superglue_zeroshot_logit} Zero-shot acc.\ on SuperGLUE with logit scoring. Best results in bold, second best underline. }
    {
        \begin{tabular}{@{}|c|cccccccc|c|@{}}
            \hline
        Model & WSC & WIC & RTE & CB & MultiRC & ReCoRD & BoolQ & COPA & Average \\ 
        \hline
        OPT-125M & \textbf{39.4} & \underline{52.0} & 48.7 & 37.4 & \underline{58.9} & \underline{44.9} & \underline{59.6} & \underline{60.0} & 50.1 \\
        GPT-Neo-125M & \underline{36.5} & \textbf{53.6} & \underline{53.1} & \underline{41.1} & \textbf{59.9} & 39.6 & \textbf{62.2} & \underline{60.0} & \underline{50.8} \\
        \textbf{Hybrid \hthree-125M} & \textbf{39.4} & 51.4 & \textbf{59.2} & \textbf{48.2} & 51.4 & \textbf{55.0} & \underline{59.6} & \textbf{67.0} & \textbf{53.9} \\ \hline 
        GPT-2 medium (355M) & \underline{50.0} & \textbf{52.0} & 51.3 & 28.6 & \textbf{59.5} & \underline{53.3} & \underline{61.0} & \underline{65.0} & 52.6 \\
        OPT-350M & \textbf{53.5} & 50.8 & \underline{53.4} & \underline{35.7} & \underline{58.9} & 51.4 & 60.9 & 60.0 & \underline{53.1} \\
        \textbf{Hybrid \hthree-355M} & 37.5 & \underline{51.7} & \textbf{55.2} & \textbf{41.1} & \textbf{59.5} & \textbf{62.3} & \textbf{61.5} & \textbf{69.0} & \textbf{54.7} \\ \hline
        OPT-1.3B & 36.5 & 49.5 & \textbf{53.4} & \textbf{39.3} & \textbf{58.3} & \underline{61.8} & 55.0 & \underline{69.0} & \underline{52.9} \\
        GPT-Neo-1.3B & \underline{41.3} & \underline{50.0} & \underline{52.3} & \underline{33.9} & 57.9 & 55.5 & \underline{59.9} & 66.0 & 52.1 \\
        \textbf{Hybrid \hthree-1.3B} & \textbf{52.9} & \textbf{50.3} & \textbf{53.4} & \underline{33.9} & \underline{58.2} & \textbf{67.8} & \textbf{61.7} & \textbf{74.0} & \textbf{56.5} \\ \hline
        OPT-2.7B & \textbf{51.0} & \underline{50.8} & 50.5 & \underline{41.1} & 57.4 & \underline{65.9} & 60.9 & 66.0 & \underline{55.5} \\
        GPT-Neo-2.7B & \underline{37.5} & 50.0 & \underline{52.3} & \textbf{50.0} & \textbf{59.1} & 60.0 & \textbf{61.1} & \underline{67.0} & 54.6 \\
        \textbf{Hybrid \hthree-2.7B} & 36.5 & \textbf{51.3} & \textbf{57.0} & 37.5 & \underline{58.7} & \textbf{71.3} & \textbf{61.1} & \textbf{81.0} & \textbf{56.8} \\ \hline
        \end{tabular}
    }
\end{table}
\begin{table}[t]
    \scriptsize
    \centering
    \caption{\label{table:superglue_fewshot_logit} 3-shot acc.\ on SuperGLUE with logit scoring. Best results in bold, second best underline. }
    {
        \begin{tabular}{@{}|c|cccccccc|c|@{}}
            \hline
        Model & WSC & WIC & RTE & CB & MultiRC & ReCoRD & BoolQ & COPA & Average \\ 
        \hline
        OPT-125M & 36.5 & \textbf{50.2} & 47.3 & \underline{44.6} & \textbf{57.9} & \underline{44.9} & 41.9 & 60.0 & \underline{47.9} \\
        GPT-Neo-125M & \underline{38.5} & \underline{50.0} & \underline{53.1} & 17.9 & \underline{56.3} & 39.6 & \textbf{62.1} & \underline{60.0} & 47.2 \\
        \textbf{Hybrid \hthree-125M} & \textbf{43.3} & 49.1 & \textbf{58.1} & \textbf{51.8} & 48.9 & \textbf{55.0} & \underline{56.1} & \textbf{67.0} & \textbf{53.7} \\ \hline 
        GPT-2 medium (355M) & 36.5 & \textbf{50.5} & \underline{48.0} & 8.9 & 43.5 & \underline{53.3} & 58.8 & \underline{65.0} & 45.6 \\
        OPT-350M & \underline{37.5} & \underline{50.0} & 45.8 & \textbf{44.6} & \underline{49.8} & 51.4 & \textbf{61.7} & 60.0 & \underline{50.1} \\
        \textbf{Hybrid \hthree-355M} & \textbf{42.3} & 47.5 & \textbf{50.5} & \underline{28.6} & \textbf{59.7} & \textbf{62.3} & \underline{60.5} & \textbf{69.0} & \textbf{52.6} \\ \hline
        OPT-1.3B & \textbf{44.2} & \textbf{51.1} & \underline{53.4} & 16.1 & \textbf{59.9} & \underline{62.1} & 38.3 & \underline{70.0} & 49.4 \\
        GPT-Neo-1.3B & 35.6 & \underline{50.6} & 47.3 & \textbf{32.1} & \textbf{59.9} & 55.7 & \textbf{61.2} & 67.0 & \underline{51.2} \\
        \textbf{Hybrid \hthree-1.3B} & \underline{36.5} & 49.2 & \textbf{55.2} & \underline{23.2} & \underline{59.3} & \textbf{67.6} & \underline{56.9} & \textbf{76.0} & \textbf{53.0} \\ \hline
        OPT-2.7B & \underline{44.2} & \underline{50.5} & \textbf{53.4} & 17.9 & \underline{59.2} & \underline{66.0} & \textbf{62.0} & \underline{71.0} & \underline{53.0} \\
        GPT-Neo-2.7B & \textbf{49.0} & \textbf{51.9} & \underline{51.6} & \underline{21.4} & 57.0 & 60.0 & 56.0 & 68.0 & 51.9 \\
        \textbf{Hybrid \hthree-2.7B} & 36.5 & 45.6 & 47.3 & \textbf{46.4} & \textbf{59.4} & \textbf{71.1} & \underline{60.6} & \textbf{77.0} & \textbf{55.5} \\ \hline
        \end{tabular}
    }
\end{table}

To understand how well capturing the synthetics in Section~\ref{sec:synthetics} translates to language modeling, we train two hybrid hybrid \hthree-attention language models at sizes 125M, 355M, 1.3B, and 2.7B, and we evaluate their performance against Transformers.
The hybrid models match or exceed the quality of Transformers in perplexity and zero/few-shot learning.
We also validate that \hthree models retain strong performance on non-text sequence modeling.
Appendix~\ref{sec:app_additional_experiments} contains additional experiments on more datasets, length extrapolation, and scaling with data.

\subsection{Language Modeling}
\label{subsec:language_modeling}

We compare hybrid \hthree-attention language models against Transformer-based language models.
We evaluate language modeling performance using perplexity, zero-shot learning, and few-shot learning performance.
Hybrid \hthree models outperform Transformers, which suggests that closing the gap between SSMs and attention on the synthetic languages translates to real language modeling capabilities.
We also report the generation speed of hybrid \hthree models compared to Transformers; since SSMs are recurrent models, they can generate tokens \num{2.4$\times$} faster than Transformers.
Appendix~\ref{sec:app_additional_experiments} shows performance of pure \hthree language models on these same evaluation metrics.

\paragraph{Setup}
We train hybrid models at sizes 125M, 355M, 1.3B, and 2.7B on the Pile~\citep{gao2020pile} for 400B tokens.
We compare against checkpoints of equivalent sizes from
Open-AI~\citep{radford2019language} and GPT-Neo\footnote{There
  is no pretrained GPT-Neo at the 350M size.}~\citep{gpt-neo},
from HuggingFace~\citep{wolf-etal-2020-transformers}.

\paragraph{Perplexity}
Table \ref{table:gpt} shows perplexity on the Pile~\citep{gao2020pile}, OpenWebText~\citep{Gokaslan2019OpenWeb}, and WikiText-103~\citep{merity2016pointer}. 
On the Pile, our 125M hybrid model outperforms GPT-Neo, which was also trained on the Pile.
Our hybrid models also outperform GPT-Neo models and GPT-2 models on zero-shot transfer to OpenWebText and WikiText103.
We report the PPL of GPT-2 models for context, though the performance is not directly comparable since they were trained on different data.

\paragraph{Zero- and Few-shot Performance}
We compare the zero- and few-shot performance of hybrid \hthree language models against OPT~\citep{zhang2022opt}, GPT-Neo, and GPT-2 models, where public checkpoints are available.
We report performance with rank classification on the logits of the possible choices (see Appendix~\ref{sec:app_generation} for raw generation).
Table~\ref{table:superglue_zeroshot_logit} reports zero-shot performance on the SuperGLUE benchmark, and Table~\ref{table:superglue_fewshot_logit} reports the 3-shot performance.
Following the perplexity results, the hybrid language models outperform or match the best the Transformer baseline on more than half the tasks.

\begin{table}[t]
\centering
    \small
    \centering
    \caption{\label{table:training_time} Inference throughput on A100 80GB, 1.3B models.
    Batch size 64, prompt length 512, 1024, or 1536, and generating 128 tokens
    per sequence in the batch (i.e., 64 $\times$ 128 tokens in a batch). Hybrid
    \hthree is up to 2.4$\times$ faster than a Transformer of similar size in inference. The
    difference is larger for longer sequences.}
    {
        \begin{tabular}{@{}|c|c|c|c|@{}}
            \hline
        Tokens/s & Prompt length 512 & Prompt length 1024 & Prompt length 1536 \\ 
        \hline
        Transformer-1.3B & 1340 & 770 & 520 \\
        Hybrid \hthree-1.3B & 1980 & 1580 & 1240 \\ \hline
        \end{tabular}
    }
\end{table}
\paragraph{Language Modeling Inference}
Finally, since SSMs are recurrent models, they admit faster text generation than Transformers.
Table~\ref{table:training_time} shows inference throughput of a 1.3B-parameter hybrid model compared to a Transformer.
The hybrid model has up to \num{2.4$\times$} higher throughput.


\section{\fastfft Evaluation \label{sec:eval_efficiency}}
We evaluate how well \fastfft speeds up SSMs.
\fastfft sets state-of-the-art performance on the long range arena benchmark~\citep{tay2020long} using S4~\citep{gu2022efficiently}.
We report performance of training \hthree module with \fastfft compared to attention at various sequence lengths, from 256 to 32K and demonstrate nearly linear scaling.

\begin{table}[t]
    \centering
    \caption{\label{table:lra} Speedup on the LRA benchmark.}
    \centering
    \small
    \begin{tabular}{|c|c|}
    \hline
    Models &  Speedup \\
    \hline
    Transformer & 1$\times$  \\
    FlashAttention~\citep{dao2022flashattention} & 2.4$\times$ \\
    Block-sparse FlashAttention~\citep{dao2022flashattention} & 2.8$\times$ \\
    \hline
    S4~\citep{gu2022train} & 2.9$\times$ \\
    S4 with \fastfft & \num{5.8$\times$} \\ \hline
    \end{tabular}
\end{table}

\paragraph{Long Range Arena}
The Long Range Arena (LRA) benchmark~\citep{tay2020long} is a benchmark for long-range sequence modeling.
The state-of-the-art approach, S4~\citep{gu2022train}, is an SSM.
Table~\ref{table:lra} shows that \fastfft accelerates S4 by 2$\times$, outperforming Transformers by 5.8$\times$.

\begin{figure}
    \centering
    \includegraphics[width=\textwidth]{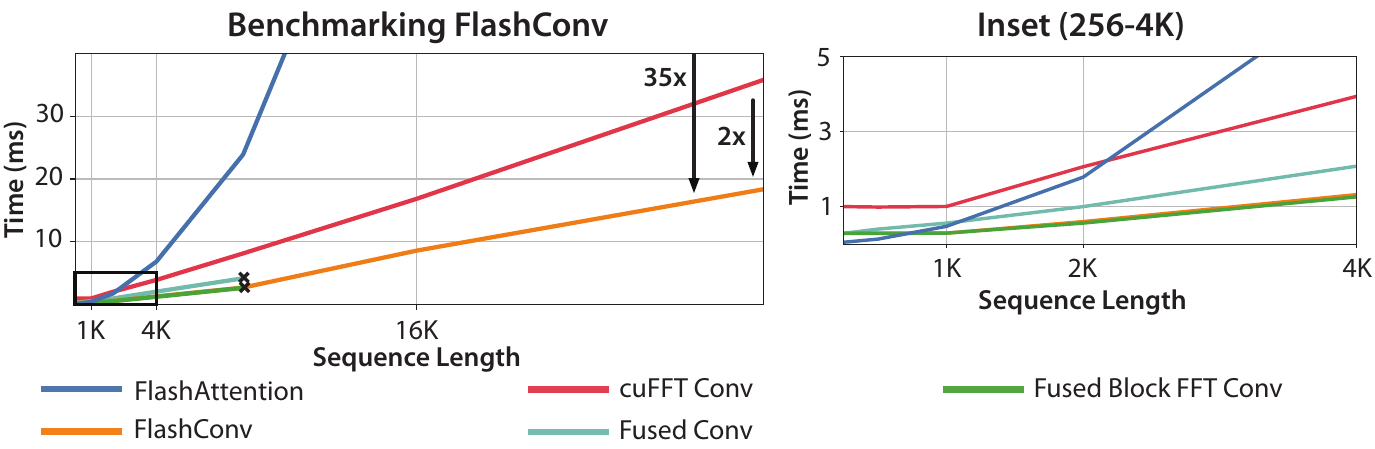}
    \caption{\label{fig:fftconv_speed}
      We compare the speed of different algorithms to perform FFT-based
      convolution, along with FlashAttention~\citep{dao2022flashattention} (the fastest attention
      implementation we know of).
      We use batch size 8, hidden dimension 1024, and varying sequence length
      from 256 to 32k, and measure on an A100-SMX4-40GB GPU.
      We see that kernel fusion gives up to 3.4$\times$ speedup over naive FFTConv
      for short sequences (up to 512), block FFT gives up to 2$\times$ speedup for
      medium length sequences (1k - 8k), and state-passing allows 2.3$\times$ faster
      FFTConv for long sequences (16k and above).
    }
\end{figure}
\paragraph{Benchmark \hthree Against Attention}
We benchmark the time to run the forward and backward pass of \hthree with \fastfft against attention.
\fastfft maintains nearly linear scaling, even to very long sequence lengths.
\cref{fig:fftconv_speed} shows overall 2-3$\times$ speedup over FFTConv with cuFFT using our techniques
(block FFT, state-passing).
Simple kernel fusion (even without block FFT) can yield speedup over cuFFT for short sequences, since memory reads/writes are the bottleneck for short sequences.
For long sequences, SSMs using state passing can be dozens of times faster
than even the fastest attention implementation.



\section{Conclusion}
\label{sec:conc}

Our main goal is to understand and narrow the gap between attention and SSMs in
language modeling in terms of modeling capabilities and hardware efficiency.
Our exploration based on synthetic language tasks motivated us to design the
\hthree layer, which is surprisingly competitive with attention.
Our \textsc{BlockFFTConv} algorithm exploits matrix multiplication units and the
dual recurrent--convolution view of SSMs to substantially speed up SSMs, reducing
the hardware barrier between attention and SSMs.
We are excited about several future directions.
Our \hthree layer is a simple combination of two SSMs, and more
sophisticated designs could be more expressive.
Our encouraging results on language models up to 1.3B parameters suggests that
scaling SSMs to larger sizes is a promising avenue.
Since simply adding two attention layers to \hthree models already
outperforms both the pure \hthree model and Transformers, we are optimistic
about combining the complementary strengths of SSMs and attention in the future.


\section*{Acknowledgments}

We thank Albert Gu for helpful discussion regarding the model architecture, and
more importantly for sending us daily hippo videos.
We thank Together Computer
for providing portions of the compute used to train models in this paper.
We gratefully acknowledge the support of NIH under No. U54EB020405 (Mobilize), NSF under Nos. CCF1763315 (Beyond Sparsity), CCF1563078 (Volume to Velocity), and 1937301 (RTML); ARL under No. W911NF-21-2-0251 (Interactive Human-AI Teaming); ONR under No. N000141712266 (Unifying Weak Supervision); ONR N00014-20-1-2480: Understanding and Applying Non-Euclidean Geometry in Machine Learning; N000142012275 (NEPTUNE); NXP, Xilinx, LETI-CEA, Intel, IBM, Microsoft, NEC, Toshiba, TSMC, ARM, Hitachi, BASF, Accenture, Ericsson, Qualcomm, Analog Devices, Google Cloud, Salesforce, Total, the HAI-GCP Cloud Credits for Research program,  the Stanford Data Science Initiative (SDSI),
Department of Defense (DoD) through the National Defense Science and
Engineering Graduate Fellowship (NDSEG) Program, 
Wu Tsai Neuroscience Stanford Interdisciplinary Graduate Fellowship,
and members of the Stanford DAWN project: Facebook, Google, and VMWare. The U.S. Government is authorized to reproduce and distribute reprints for Governmental purposes notwithstanding any copyright notation thereon. Any opinions, findings, and conclusions or recommendations expressed in this material are those of the authors and do not necessarily reflect the views, policies, or endorsements, either expressed or implied, of DARPA, NIH, ONR, or the U.S. Government.
Atri Rudra’s research is supported by NSF grant CCF-1763481.

\ifarxiv
\bibliographystyle{plain} 
\fi
\bibliography{main}

\begin{thebibliography}{10}

\bibitem{ba2016using}
Jimmy Ba, Geoffrey~E Hinton, Volodymyr Mnih, Joel~Z Leibo, and Catalin Ionescu.
\newblock Using fast weights to attend to the recent past.
\newblock {\em Advances in neural information processing systems}, 29, 2016.

\bibitem{baevski2018adaptive}
Alexei Baevski and Michael Auli.
\newblock Adaptive input representations for neural language modeling.
\newblock In {\em International Conference on Learning Representations}, 2018.

\bibitem{bailey1990ffts}
David~H Bailey.
\newblock {FFT}s in external or hierarchical memory.
\newblock {\em The journal of Supercomputing}, 4(1):23--35, 1990.

\bibitem{gpt-neo}
Sid Black, Leo Gao, Phil Wang, Connor Leahy, and Stella Biderman.
\newblock {GPT-Neo: Large Scale Autoregressive Language Modeling with
  Mesh-Tensorflow}, March 2021.
\newblock {If you use this software, please cite it using these metadata.}

\bibitem{bommasani2021opportunities}
Rishi Bommasani, Drew~A Hudson, Ehsan Adeli, Russ Altman, Simran Arora, Sydney
  von Arx, Michael~S Bernstein, Jeannette Bohg, Antoine Bosselut, Emma
  Brunskill, et~al.
\newblock On the opportunities and risks of foundation models.
\newblock {\em arXiv preprint arXiv:2108.07258}, 2021.

\bibitem{brogan1974modern}
Willian~L Brogan.
\newblock Modern control theory, 1974.

\bibitem{brown2020language}
Tom Brown, Benjamin Mann, Nick Ryder, Melanie Subbiah, Jared~D Kaplan, Prafulla
  Dhariwal, Arvind Neelakantan, Pranav Shyam, Girish Sastry, Amanda Askell,
  et~al.
\newblock Language models are few-shot learners.
\newblock {\em Advances in neural information processing systems},
  33:1877--1901, 2020.

\bibitem{cho2014properties}
Kyunghyun Cho, Bart Van~Merri{\"e}nboer, Dzmitry Bahdanau, and Yoshua Bengio.
\newblock On the properties of neural machine translation: Encoder-decoder
  approaches.
\newblock {\em arXiv preprint arXiv:1409.1259}, 2014.

\bibitem{choromanski2020rethinking}
Krzysztof~Marcin Choromanski, Valerii Likhosherstov, David Dohan, Xingyou Song,
  Andreea Gane, Tamas Sarlos, Peter Hawkins, Jared~Quincy Davis, Afroz
  Mohiuddin, Lukasz Kaiser, et~al.
\newblock Rethinking attention with performers.
\newblock In {\em International Conference on Learning Representations (ICLR)},
  2020.

\bibitem{chowdhery2022palm}
Aakanksha Chowdhery, Sharan Narang, Jacob Devlin, Maarten Bosma, Gaurav Mishra,
  Adam Roberts, Paul Barham, Hyung~Won Chung, Charles Sutton, Sebastian
  Gehrmann, et~al.
\newblock Palm: Scaling language modeling with pathways.
\newblock {\em arXiv preprint arXiv:2204.02311}, 2022.

\bibitem{cooley1965an}
James~W. Cooley and John~W. Tukey.
\newblock An algorithm for the machine calculation of complex fourier series.
\newblock {\em Mathematics of Computation}, 19(90):297--301, 1965.

\bibitem{dadi_2020_fine}
Kamalaker Dadi, Ga{\"e}l Varoquaux, Antonia Machlouzarides-Shalit, Krzysztof~J
  Gorgolewski, Demian Wassermann, Bertrand Thirion, and Arthur Mensch.
\newblock Fine-grain atlases of functional modes for fmri analysis.
\newblock {\em NeuroImage}, 221:117126, 2020.

\bibitem{dai2019transformer}
Zihang Dai, Zhilin Yang, Yiming Yang, Jaime~G Carbonell, Quoc Le, and Ruslan
  Salakhutdinov.
\newblock Transformer-xl: Attentive language models beyond a fixed-length
  context.
\newblock In {\em Proceedings of the 57th Annual Meeting of the Association for
  Computational Linguistics}, pages 2978--2988, 2019.

\bibitem{dao2022monarch}
Tri Dao, Beidi Chen, Nimit Sohoni, Arjun Desai, Michael Poli, Jessica Grogan,
  Alexander Liu, Aniruddh Rao, Atri Rudra, and Christopher R{\'e}.
\newblock Monarch: Expressive structured matrices for efficient and accurate
  training.
\newblock In {\em International Conference on Machine Learning (ICML)}, 2022.

\bibitem{dao2022flashattention}
Tri Dao, Daniel~Y. Fu, Stefano Ermon, Atri Rudra, and Christopher R{\'e}.
\newblock Flashattention: Fast and memory-efficient exact attention with
  io-awareness.
\newblock In {\em Advances in Neural Information Processing Systems}, 2022.

\bibitem{daras2020smyrf}
Giannis Daras, Nikita Kitaev, Augustus Odena, and Alexandros~G Dimakis.
\newblock Smyrf-efficient attention using asymmetric clustering.
\newblock {\em Advances in Neural Information Processing Systems},
  33:6476--6489, 2020.

\bibitem{edelman2022inductive}
Benjamin~L Edelman, Surbhi Goel, Sham Kakade, and Cyril Zhang.
\newblock Inductive biases and variable creation in self-attention mechanisms.
\newblock In {\em International Conference on Machine Learning}, pages
  5793--5831. PMLR, 2022.

\bibitem{elhage2021mathematical}
Nelson Elhage, Neel Nanda, Catherine Olsson, Tom Henighan, Nicholas Joseph, Ben
  Mann, Amanda Askell, Yuntao Bai, Anna Chen, Tom Conerly, Nova DasSarma, Dawn
  Drain, Deep Ganguli, Zac Hatfield-Dodds, Danny Hernandez, Andy Jones, Jackson
  Kernion, Liane Lovitt, Kamal Ndousse, Dario Amodei, Tom Brown, Jack Clark,
  Jared Kaplan, Sam McCandlish, and Chris Olah.
\newblock A mathematical framework for transformer circuits.
\newblock {\em Transformer Circuits Thread}, 2021.
\newblock https://transformer-circuits.pub/2021/framework/index.html.

\bibitem{fischl_2012_freesurfer}
Bruce Fischl.
\newblock Freesurfer.
\newblock {\em Neuroimage}, 62(2):774--781, 2012.

\bibitem{fisher2014ilae}
Robert~S Fisher, Carlos Acevedo, Alexis Arzimanoglou, Alicia Bogacz, J~Helen
  Cross, Christian~E Elger, Jerome Engel~Jr, Lars Forsgren, Jacqueline~A
  French, Mike Glynn, et~al.
\newblock Ilae official report: a practical clinical definition of epilepsy.
\newblock {\em Epilepsia}, 55(4):475--482, 2014.

\bibitem{gao2020pile}
Leo Gao, Stella Biderman, Sid Black, Laurence Golding, Travis Hoppe, Charles
  Foster, Jason Phang, Horace He, Anish Thite, Noa Nabeshima, et~al.
\newblock The pile: An 800gb dataset of diverse text for language modeling.
\newblock {\em arXiv preprint arXiv:2101.00027}, 2020.

\bibitem{goel2022s}
Karan Goel, Albert Gu, Chris Donahue, and Christopher R{\'e}.
\newblock It's raw! audio generation with state-space models.
\newblock {\em arXiv preprint arXiv:2202.09729}, 2022.

\bibitem{Gokaslan2019OpenWeb}
Aaron Gokaslan, Vanya Cohen, Ellie Pavlick, and Stefanie Tellex.
\newblock Openwebtext corpus, 2019.

\bibitem{gu2020hippo}
Albert Gu, Tri Dao, Stefano Ermon, Atri Rudra, and Christopher R{\'e}.
\newblock Hippo: Recurrent memory with optimal polynomial projections.
\newblock {\em Advances in Neural Information Processing Systems},
  33:1474--1487, 2020.

\bibitem{gu2022efficiently}
Albert Gu, Karan Goel, and Christopher R\'e.
\newblock Efficiently modeling long sequences with structured state spaces.
\newblock In {\em The International Conference on Learning Representations
  ({ICLR})}, 2022.

\bibitem{gu2022parameterization}
Albert Gu, Ankit Gupta, Karan Goel, and Christopher R{\'e}.
\newblock On the parameterization and initialization of diagonal state space
  models.
\newblock In {\em Advances in Neural Information Processing Systems}, 2022.

\bibitem{gu2021combining}
Albert Gu, Isys Johnson, Karan Goel, Khaled Saab, Tri Dao, Atri Rudra, and
  Christopher R{\'e}.
\newblock Combining recurrent, convolutional, and continuous-time models with
  linear state-space layers.
\newblock {\em Advances in neural information processing systems}, 34, 2021.

\bibitem{gu2022train}
Albert Gu, Isys Johnson, Aman Timalsina, Atri Rudra, and Christopher R{\'e}.
\newblock How to train your hippo: State space models with generalized
  orthogonal basis projections.
\newblock {\em arXiv preprint arXiv:2206.12037}, 2022.

\bibitem{gupta2022diagonal}
Ankit Gupta, Albert Gu, and Jonathan Berant.
\newblock Diagonal state spaces are as effective as structured state spaces.
\newblock In {\em Advances in Neural Information Processing Systems}, 2022.

\bibitem{hawthorne2022general}
Curtis Hawthorne, Andrew Jaegle, C{\u{a}}t{\u{a}}lina Cangea, Sebastian
  Borgeaud, Charlie Nash, Mateusz Malinowski, Sander Dieleman, Oriol Vinyals,
  Matthew Botvinick, Ian Simon, et~al.
\newblock General-purpose, long-context autoregressive modeling with perceiver
  ar.
\newblock {\em arXiv preprint arXiv:2202.07765}, 2022.

\bibitem{hochreiter1996lstm}
Sepp Hochreiter and J{\"u}rgen Schmidhuber.
\newblock Lstm can solve hard long time lag problems.
\newblock {\em Advances in neural information processing systems}, 9, 1996.

\bibitem{hoffmann2022training}
Jordan Hoffmann, Sebastian Borgeaud, Arthur Mensch, Elena Buchatskaya, Trevor
  Cai, Eliza Rutherford, Diego de~Las Casas, Lisa~Anne Hendricks, Johannes
  Welbl, Aidan Clark, et~al.
\newblock Training compute-optimal large language models.
\newblock {\em arXiv preprint arXiv:2203.15556}, 2022.

\bibitem{hooker2021hardware}
Sara Hooker.
\newblock The hardware lottery.
\newblock {\em Communications of the ACM}, 64(12):58--65, 2021.

\bibitem{kao2021optimized}
Sheng-Chun Kao, Suvinay Subramanian, Gaurav Agrawal, and Tushar Krishna.
\newblock An optimized dataflow for mitigating attention performance
  bottlenecks.
\newblock {\em arXiv preprint arXiv:2107.06419}, 2021.

\bibitem{katharopoulos2020transformers}
Angelos Katharopoulos, Apoorv Vyas, Nikolaos Pappas, and Fran{\c{c}}ois
  Fleuret.
\newblock Transformers are {RNN}s: Fast autoregressive transformers with linear
  attention.
\newblock In {\em International Conference on Machine Learning}, pages
  5156--5165. PMLR, 2020.

\bibitem{keles2022computational}
Feyza~Duman Keles, Pruthuvi~Mahesakya Wijewardena, and Chinmay Hegde.
\newblock On the computational complexity of self-attention.
\newblock {\em arXiv preprint arXiv:2209.04881}, 2022.

\bibitem{kerr2012impact}
Michael~Patrick Kerr.
\newblock The impact of epilepsy on patients' lives.
\newblock {\em Acta Neurologica Scandinavica}, 126:1--9, 2012.

\bibitem{king_2019_functional}
Maedbh King, Carlos~R Hernandez-Castillo, Russell~A Poldrack, Richard~B Ivry,
  and J{\"o}rn Diedrichsen.
\newblock Functional boundaries in the human cerebellum revealed by a
  multi-domain task battery.
\newblock {\em Nature neuroscience}, 22(8):1371--1378, 2019.

\bibitem{kitaev2020reformer}
Nikita Kitaev, {\L}ukasz Kaiser, and Anselm Levskaya.
\newblock Reformer: The efficient transformer.
\newblock In {\em The International Conference on Machine Learning ({ICML})},
  2020.

\bibitem{luo2021stable}
Shengjie Luo, Shanda Li, Tianle Cai, Di~He, Dinglan Peng, Shuxin Zheng, Guolin
  Ke, Liwei Wang, and Tie-Yan Liu.
\newblock Stable, fast and accurate: Kernelized attention with relative
  positional encoding.
\newblock {\em Advances in Neural Information Processing Systems},
  34:22795--22807, 2021.

\bibitem{markiewicz_2021_openneuro}
Christopher~J Markiewicz, Krzysztof~J Gorgolewski, Franklin Feingold, Ross
  Blair, Yaroslav~O Halchenko, Eric Miller, Nell Hardcastle, Joe Wexler, Oscar
  Esteban, Mathias Goncavles, et~al.
\newblock The openneuro resource for sharing of neuroscience data.
\newblock {\em Elife}, 10:e71774, 2021.

\bibitem{mehta2022long}
Harsh Mehta, Ankit Gupta, Ashok Cutkosky, and Behnam Neyshabur.
\newblock Long range language modeling via gated state spaces.
\newblock {\em arXiv preprint arXiv:2206.13947}, 2022.

\bibitem{merity2016pointer}
Stephen Merity, Caiming Xiong, James Bradbury, and Richard Socher.
\newblock Pointer sentinel mixture models, 2016.

\bibitem{nguyen2022s4nd}
Eric Nguyen, Karan Goel, Albert Gu, Gordon Downs, Preey Shah, Tri Dao, Stephen
  Baccus, and Christopher R{\'e}.
\newblock S4nd: Modeling images and videos as multidimensional signals with
  state spaces.
\newblock In {\em Advances in Neural Information Processing Systems}, 2022.

\bibitem{nvidia2017nvidia}
NVIDIA.
\newblock Nvidia {T}esla {V}100 {GPU} architecture, 2017.

\bibitem{nvidia2020nvidia}
NVIDIA.
\newblock Nvidia {A}100 tensor core {GPU} architecture, 2020.

\bibitem{cufft}
NVIDIA.
\newblock cufft v11.7.1 documentation, 2022.
\newblock https://docs.nvidia.com/cuda/cufft/index.html.

\bibitem{nvidia2022nvidia}
NVIDIA.
\newblock Nvidia {H}100 tensor core {GPU} architecture, 2022.

\bibitem{olsson2022context}
Catherine Olsson, Nelson Elhage, Neel Nanda, Nicholas Joseph, Nova DasSarma,
  Tom Henighan, Ben Mann, Amanda Askell, Yuntao Bai, Anna Chen, Tom Conerly,
  Dawn Drain, Deep Ganguli, Zac Hatfield-Dodds, Danny Hernandez, Scott
  Johnston, Andy Jones, Jackson Kernion, Liane Lovitt, Kamal Ndousse, Dario
  Amodei, Tom Brown, Jack Clark, Jared Kaplan, Sam McCandlish, and Chris Olah.
\newblock In-context learning and induction heads.
\newblock {\em Transformer Circuits Thread}, 2022.
\newblock
  https://transformer-circuits.pub/2022/in-context-learning-and-induction-heads/index.html.

\bibitem{oppenheim1978applications}
Alan~V Oppenheim.
\newblock Applications of digital signal processing.
\newblock {\em Englewood Cliffs}, 1978.

\bibitem{oppenheim2001discrete}
Alan~V Oppenheim, John~R Buck, and Ronald~W Schafer.
\newblock {\em Discrete-time signal processing. Vol. 2}.
\newblock Upper Saddle River, NJ: Prentice Hall, 2001.

\bibitem{rabe2021self}
Markus~N Rabe and Charles Staats.
\newblock Self-attention does not need $ {O} (n^2) $ memory.
\newblock {\em arXiv preprint arXiv:2112.05682}, 2021.

\bibitem{radford2019language}
Alec Radford, Jeffrey Wu, Rewon Child, David Luan, Dario Amodei, Ilya
  Sutskever, et~al.
\newblock Language models are unsupervised multitask learners.
\newblock {\em OpenAI blog}, 1(8):9, 2019.

\bibitem{rae2019compressive}
Jack~W Rae, Anna Potapenko, Siddhant~M Jayakumar, Chloe Hillier, and Timothy~P
  Lillicrap.
\newblock Compressive transformers for long-range sequence modelling.
\newblock In {\em International Conference on Learning Representations}, 2019.

\bibitem{saab2020weak}
Khaled Saab, Jared Dunnmon, Christopher R{\'e}, Daniel Rubin, and Christopher
  Lee-Messer.
\newblock Weak supervision as an efficient approach for automated seizure
  detection in electroencephalography.
\newblock {\em NPJ digital medicine}, 3(1):1--12, 2020.

\bibitem{shah2018temple}
Vinit Shah, Eva Von~Weltin, Silvia Lopez, James~Riley McHugh, Lillian Veloso,
  Meysam Golmohammadi, Iyad Obeid, and Joseph Picone.
\newblock The temple university hospital seizure detection corpus.
\newblock {\em Frontiers in neuroinformatics}, 12:83, 2018.

\bibitem{siddiqui2020review}
Mohammad~Khubeb Siddiqui, Ruben Morales-Menendez, Xiaodi Huang, and Nasir
  Hussain.
\newblock A review of epileptic seizure detection using machine learning
  classifiers.
\newblock {\em Brain informatics}, 7(1):1--18, 2020.

\bibitem{tang2021self}
Siyi Tang, Jared Dunnmon, Khaled~Kamal Saab, Xuan Zhang, Qianying Huang,
  Florian Dubost, Daniel Rubin, and Christopher Lee-Messer.
\newblock Self-supervised graph neural networks for improved
  electroencephalographic seizure analysis.
\newblock In {\em International Conference on Learning Representations}, 2021.

\bibitem{tay2020long}
Yi~Tay, Mostafa Dehghani, Samira Abnar, Yikang Shen, Dara Bahri, Philip Pham,
  Jinfeng Rao, Liu Yang, Sebastian Ruder, and Donald Metzler.
\newblock Long range arena: A benchmark for efficient transformers.
\newblock In {\em International Conference on Learning Representations}, 2020.

\bibitem{thomas_fmri_2022}
Armin~W Thomas, Christopher R{\'e}, and Russell~A Poldrack.
\newblock Self-supervised learning of brain dynamics from broad neuroimaging
  data.
\newblock {\em arXiv preprint arXiv:2206.11417}, 2022.

\bibitem{van_2013_wu}
David~C Van~Essen, Stephen~M Smith, Deanna~M Barch, Timothy~EJ Behrens, Essa
  Yacoub, Kamil Ugurbil, Wu-Minn~HCP Consortium, et~al.
\newblock The wu-minn human connectome project: an overview.
\newblock {\em Neuroimage}, 80:62--79, 2013.

\bibitem{vaswani2017attention}
Ashish Vaswani, Noam Shazeer, Niki Parmar, Jakob Uszkoreit, Llion Jones,
  Aidan~N Gomez, {\L}ukasz Kaiser, and Illia Polosukhin.
\newblock Attention is all you need.
\newblock {\em Advances in neural information processing systems}, 30, 2017.

\bibitem{wang2020linformer}
Sinong Wang, Belinda~Z Li, Madian Khabsa, Han Fang, and Hao Ma.
\newblock Linformer: Self-attention with linear complexity.
\newblock {\em arXiv preprint arXiv:2006.04768}, 2020.

\bibitem{warden2018speech}
Pete Warden.
\newblock Speech commands: A dataset for limited-vocabulary speech recognition.
\newblock {\em arXiv preprint arXiv:1804.03209}, 2018.

\bibitem{wolf-etal-2020-transformers}
Thomas Wolf, Lysandre Debut, Victor Sanh, Julien Chaumond, Clement Delangue,
  Anthony Moi, Pierric Cistac, Tim Rault, Rémi Louf, Morgan Funtowicz, Joe
  Davison, Sam Shleifer, Patrick von Platen, Clara Ma, Yacine Jernite, Julien
  Plu, Canwen Xu, Teven~Le Scao, Sylvain Gugger, Mariama Drame, Quentin Lhoest,
  and Alexander~M. Rush.
\newblock Transformers: State-of-the-art natural language processing.
\newblock In {\em Proceedings of the 2020 Conference on Empirical Methods in
  Natural Language Processing: System Demonstrations}, pages 38--45, Online,
  October 2020. Association for Computational Linguistics.

\bibitem{zhang2022opt}
Susan Zhang, Stephen Roller, Naman Goyal, Mikel Artetxe, Moya Chen, Shuohui
  Chen, Christopher Dewan, Mona Diab, Xian Li, Xi~Victoria Lin, et~al.
\newblock {OPT}: Open pre-trained transformer language models.
\newblock {\em arXiv preprint arXiv:2205.01068}, 2022.

\end{thebibliography}

\appendix
\newpage

\section{Related Work}
\label{sec:related}

\textbf{State space models} have shown promise in modeling sequential data, including time series data~\citep{gu2022efficiently}, audio~\citep{goel2022s}, and visual data~\citep{nguyen2022s4nd}.
Our model builds off work on simplifying and parameterizing diagonal versions of S4~\citep{gu2022parameterization,gupta2022diagonal, gu2022train}.
Gated state spaces~\citep{mehta2022long} also aim to adapt SSMs to language modeling, but our results suggest that the GSS model does not perform as well as \hthree (or even as well as earlier SSMs like S4D).
The idea to combine SSMs with attention in hybrid models is not new; Mehta et al.~\citep{mehta2022long} also showed that interleaving attention with their GSS layer can improve performance, which we also validate on our OpenWebText experiments.
These positive results suggest that attention and SSMs are complementary, and that hybrid models may be a promising direction for future work.

\textbf{Large language foundation models}~\citep{bommasani2021opportunities} have demonstrated the power of scaling attention-based networks to billions of parameters and training them on trillions of tokens~\citep{hoffmann2022training}.
Understanding the mechanistic basis~\citep{elhage2021mathematical} behind these models may yield insights into better design choices for future models.
These and similar explorations have informed the design of \hthree and our selection of synthetic languages.
A number of recent works have also explored how to address the shortcomings of attention by approximating the attention computation~\citep{wang2020linformer,katharopoulos2020transformers, choromanski2020rethinking,tay2020long, kitaev2020reformer, daras2020smyrf}.
We believe these efforts are complementary to SSMs, and we are excited to see how they can be combined in future work.

\textbf{Linear attention}~\citep{katharopoulos2020transformers} and classical sequence models like RNNs serve as inspiration for \hthree.
Appendix~\ref{app:linear_attention} draws a direct connection between linear attention and LTI systems.
Luo et al.~\citep{luo2021stable} also propose a variant of linear attention that can achieve $O(n \log n)$ scaling in sequence length.
Appendix~\ref{sec:app_additional_experiments} evaluates linear attention on language modeling, and finds that it underperforms exact attention, whereas \hthree outperforms attention.
The multiplicative interactions in \hthree are reminiscent of gating mechanisms in LSTMs~\citep{hochreiter1996lstm} and GRUs~\citep{cho2014properties}, which suggests that architectural lessons from these sequence models may be useful for adapting SSMs to language modeling.
A number of algorithms for scaling attention to longer sequences have also been proposed, such as Transformer-XL~\citep{dai2019transformer}, Reformer~\citep{kitaev2020reformer}, Performer~\citep{choromanski2020rethinking}, and Perceiver AR~\citep{hawthorne2022general}.
Some of these approaches underperform exact attention on language modeling, and may be slower in wall-clock speed~\citep{dao2022flashattention}.
A thorough comparison of these alternatives to exact attention and how well they scale in model size and amount of training data is fruitful future work.

\textbf{FFT} algorithms are used in a wide variety of applications, including signal processing~\citep{oppenheim1978applications}, control theory~\citep{brogan1974modern}, and more.
Various algorithms for computing the FFT have existed for decades~\citep{oppenheim2001discrete}.
We hope our work on appealing to these classic algorithms to accelerate new applications such as learned SSMs will inspire future algorithmic exploration, even if hardware is not designed for them~\citep{hooker2021hardware}.


\section{Linear Attention and Time-Varying Systems}
\label{app:linear_attention}

We draw some connections from linear attention to LTI systems and SSMs.

We first present linear attention as a linear time-varying system, and show how converting it to a linear time-invariant system matches \hthree.

\paragraph{Linear time-varying system and linear attention}

In general, a layer in a sequence model takes in a sequence and outputs a
sequence.
Many of these take the form of a linear time-varying system (thanks to the
Picard-Lindelof theorem, nonlinear systems can be approximated by a series of
linear system):
\begin{align*}
  x_i &= \vA_i x_{i-1} + \vB_i u_i, \\
  y_i &= \vC_i x_i + \vD_i u_i.
\end{align*}
This has the same form as SSMs (\cref{sec:background}), except that the matrices
can depend on the timestep.

Recall the form of linear attention from~\cref{sec:background}.
For the purpose of approximation, we ignore the denominator in linear
attention~\cref{sec:background} (i.e., assuming $d_i = 1$).
We see that $S_i$ is just a cumulative sum, satisfying the recurrence $S_{i+1} = S_{i} + \phi(K_{i+1}) V_{i+1}^T$.
Similarly, $O_i$ satisfies the recurrence $O_{i+1} = \phi(Q_{i+1})^T S_{i+1}$.
This is a linear time-varying system of the form $x_{i+1} = \vA x_i + \vB u_{i+1}$ and
$y_{i+1} = \vC_{i+1} x_{i+1}$ (with $\vA = I$, $\vB = I$, $u_{i} = \phi(K_{i}) V_i^T$,
$C_{i} = \phi(Q_{i})^T$).
That is, $\vA$ and $\vB$ are constant, but $C$ is time-variant.

To convert this into a linear time-invariant version, we treat the time-variant $\vC_i$ as a post-processing step.
We instead of a fixed $\vC$ for the LTI.
This yields an LTI:
\begin{align*}
  x_{i+1} &= \vA x_i + \vB \phi(K_{i}) V_i^T, \\
  y_{i+1} &= \vC x_{i},
\end{align*}
for some matrices $\vA, \vB, \vC$ that are learned.
We then apply post-processing by multiply $y_{i+1}$ with $\phi(Q_i)^T$.
Replacing $\phi(K_{i})$ with a shift SSM yields an analogue to \hthree.

\section{Method details}
\label{sec:method_details}

Since we have described the forward pass in~\cref{sec:method}, we describe here
the backward pass in details.

\subsection{Backward Pass}
We show how to compute the backward pass in a fused kernel.

Let $y = f \ast u + \vD u$.
In our case, we have $f$ and $u$ have the same length, so they are symmetric as far as the convolution is concerned.

Suppose we are given $dy = \frac{\partial l}{\partial y}$ (where $l$ is some loss function).
We wish to compute $du$, $df$, and $dD$ (which are $\frac{\partial l}{\partial u}$, $\frac{\partial l}{\partial f}$, and $\frac{\partial l}{\partial \vD}$, respectively).

The most challenging part is computing the gradient through the convolution operator - but we'll see that we can re-use our FFT infrastructure for it.
The rest of the operations are straightforward; we have $d\vD = dy u^T$.

\paragraph{Gradient of the Convolution}

Here, we'll discuss how to compute $df$ by integrating w.r.t. the convolution operator $\ast$.
As an immediate consequence, we'll be able to compute $du$ as well.

Since $f$ and $u$ are the same length $L$, $f \ast u$ and $u \ast f$ have the same result.
Thus, we'll start from $u \ast f$ here.

For some notation, let $O = u \ast f$.
Then, $dO = dy$.
Recall that $O[i] = \sum_{j=0}^{i-1} u[i-j]f[j]$.

We'll start by extending $u$ and $f$ with zeros, to give them length $2L$.
Let $u' = [u[0], u[1], \dots, u[L-1], 0, \dots, 0]$, and $f'$ extended similarly.
Let $O' = u' \ast f'$, and $O = O'[:N]$.
Assume that we have all the values of $dO'$ (we only have them for the first half, but we'll see that it doesn't matter in the end).

Let's construct a Toeplitz matrix $H_{u'}$ such that $u' \ast f' = H_{u'} f'$:
$$
H_{u'} = \begin{bmatrix}
  u'[0] & 0     & \dots & 0 \\
  u'[1] & u'[0] & \dots & 0 \\
  \vdots & \vdots & \ddots & \vdots \\
  u'[2L-1] & u'[2L-2] & \dots & u'[0]
\end{bmatrix}
$$
Since, we have $u'[i] = f'[i] = 0$ for $i \geq L$, we can actually fill in the zeros of the above matrix as well:
$$
H_{u'} = \begin{bmatrix}
  u'[0] & u'[2L-1] & \dots & u'[1] \\
  u'[1] & u'[0] & \dots & u'[2] \\
  \vdots & \vdots & \ddots & \vdots \\
  u'[2L-1] & u'[2L-2] & \dots & u'[0]
\end{bmatrix}
$$
Then, we can use the matrix multiplication chain rule to find that:
\begin{align*}
df' = H_{u'}^T dO' &= \begin{bmatrix}
  u'[0] & u'[1] & \dots & u'[2L-1] \\
  u'[2L-1] & u'[0] & \dots & u'[2L-2] \\
  \vdots & \vdots & \ddots & \vdots \\
  u'[1] & u'[2] & \dots & u'[0]
\end{bmatrix} \\
&= \begin{bmatrix}
  u'[0] & u'[-(2L-1)] & \dots & u'[-1] \\
  u'[-1] & u'[0] & \dots & u'[-2] \\
  \vdots & \vdots & \ddots & \vdots \\
  u'[-(2L-1)] & u'[-(2L-2)] & \dots & u'[0]
\end{bmatrix},
\end{align*}
where we use $u'[-i]$ to mean $u'[2L - i]$.
Notice that this matrix has the same format as $H_{u'}$!
Let $u_*' = [u'[0], u'[-1], \dots, u'[-(2N-1)]]$.
Then:
$$
df' = (u_*' \ast dO').
$$
So how do we compute $u_*'$ efficiently?
Naively, we might incur some nasty memory access issues.
But a nice property about the DFT saves us!

Let $U[i]$ be the $i$-th element of the DFT of a signal $u$.
Note that $U[i]$ is complex.
We have:
$$
U^*[i] = U[-i],
$$
where here the $*$ represents the complex conjugate.
We can use this property to compute $df'$ efficiently:
$$df' = u_*' \ast dO'= iFFT(FFT^*(u')FFT(dO')) \Rightarrow df = df'[:N] = iFFT(FFT^*(u')FFT(dy'))[:N],$$
where $FFT^*$ denotes taking the complex conjugate of the FFT, and $dy'$ denotes $dy$, padded with zeros.

\paragraph{Computing $du$}
We can use this same trick to compute $du$, except we need to add in the contribution from the original $\vD u$ term.
We end up with:
$$du = du'[:N] + \vD dy = iFFT(FFT^*(f')FFT(dy'))[:N] + \vD dy.$$

\subsection{State-Passing Matrices}
\label{sec:state-passing-matrices}

We show how to derive $\vM_{ux}$ for the state update in our state-passing algorithm.

We wish to construct a matrix $vM_{ux} \in \mathbb{R}^{m \times N'}$ such that $\vM_{ux}u = \sum_{i=1}^{N'} \vA^{N'-1}\vB u_i$.
Note that $\vA^i \vB \in \mathbb{R}^{d \times 1}$ is a column vector.
We can simply stack these column vectors to form a matrix:
$\vM_{ux} = [\vA^{N'-1}\vB, \vA^{N'-2}\vB, \dots, \vB]$.

\section{Proofs}
\label{sec:proofs}

We show parameterizations of \hthree and attention that solves the associative recall task.
We prove~\cref{thm:h3_complexity} and~\cref{thm:statepassing_correctness}.

\subsection{\hthree Expressivity}
\label{sec:app_expressivity}

This section formally describes a parameterization of \hthree that solves the associative recall task.

\subsubsection{Example Language $\Lambda$}
Consider a simple language with 4 keys and 4 values.
For concreteness, we will use the keys $\{k_1, k_2, k_3, k_4\} = L_K$ and the values $\{v_1, v_2, v_3, v_4\} = L_V$, i.e. our language $L = L_K \cup L_V$.
Given a sequence of key-value pairs with one key at the end, we want a model to generate the value associated with the key at the end.
Assume that the key at the end appeared in the sequence.

More formally, let $N+1$ be the length of the sequence, $N$ even.
The language $\Lambda$ consists of sequences $x \in L^{N+1}$.
Each sequence has an associated mapping $f_x : L_K \rightarrow L_V$.
For each sequence, the odd indices are randomly sampled from $L_K$, for $x_1, x_3, \dots, x_{N-1}$.
The even indices are defined by $f_x$: $x_{2*i} = f_x(x_{2*i-1})$, for $1 \leq i \leq N / 2$.
The last item in the sequence $x_{N+1}$, is randomly drawn from the keys that have appeared in $x$ already, i.e. $x_{N+1} \in \cup{\{x_1, x_3, \dots, x_{N-1}\}}$.
The goal of this language modeling task is to produce $f_x(x_{N+1})$ at the end of the sequence.

\subsubsection{H3 Model to Solve $\Lambda$}
We describe a toy \hthree model that can solve $\Lambda$.

Consider a model consisting of an embedding layer, an \hthree model, and an output projection with softmax.
Recall that $d$ is the dimension of the \hthree model, $m$ is the dimension of its hidden states, and $H$ is the number of heads.
Let $d = 8, m = 2, H = 4$.
Let the embedding layer map each key $k_i$ to the $e_i$ basis vector, and map each value $v_i$ to the $e_{4+i}$ basis vector.

Let $\vB_{shift}$ and $\vC_{shift}$ denote the parameters of the shift SSM, and $\vA_{diag}$, $\vB_{diag}$, and $\vC_{diag}$ denote the parameters of the diagonal SSM (let $\vD$ be zero for both).
Let $\vB_{shift} = \vB_{diag} = \vC_{diag} = e_1$.
Let $\vC_{shift}$ = $[0 1]$.
Let $\vA_{diag}$ be a diagonal matrix with $1$s along its diagonal for each \hthree.

\begin{remark}
The action of a diagonal SSM parameterized by $\vA_{diag}$, $\vB_{diag}$, and $\vC_{diag}$ is to act as a cumulative sum over all its input.
The action of shift SSM parameterized by $\vB_{shift}$ and $\vC_{shift}$ is to shift its input by one time step.
\end{remark}

Recall that the \hthree layer maps its input to $Q$, $K$, and $V$ by applying $u\vW_Q$, $u\vW_K$, and $u\vW_V$.
Let $\vW_Q$ and $\vW_K$ be the following:
$$
    \vW_Q = \vW_K = \begin{bmatrix}
        1 & 1 & 0 & 0 & 0 & 0 & 0 & 0 \\
        0 & 0 & 1 & 1 & 0 & 0 & 0 & 0 \\
        0 & 0 & 0 & 0 & 1 & 1 & 0 & 0 \\
        0 & 0 & 0 & 0 & 0 & 0 & 1 & 1 \\
        0 & 0 & 0 & 0 & 0 & 0 & 0 & 0 \\
        0 & 0 & 0 & 0 & 0 & 0 & 0 & 0 \\
        0 & 0 & 0 & 0 & 0 & 0 & 0 & 0 \\
        0 & 0 & 0 & 0 & 0 & 0 & 0 & 0 \\
    \end{bmatrix}
$$
Recall that $Q$ and $K$ are split into $H$ heads ($\vQ^{(i)}, \vK^{(i)}$ for $i \in \{1, 2, 3, 4\}$), each of which is sent to an independent \hthree.

\begin{remark}
The action of $\vW_Q$ and $\vW_K$ are to ``assign'' each key to a different \hthree head, i.e., $\vQ^{(i)}_t$ is only non-zero when $x_t = k_i$.
Similarly, $\overline{\vK}^{(i)}_t$ is only non-zero when $x_{t-1} = k_i$ (since $\overline{\vK}_t = \vK_{t-1}$ due to the time delay of the shift SSM).
\end{remark}

Let $\vW_V$ be the following:
$$
    \vW_V = \begin{bmatrix}
        0 & 0 & 0 & 0 & 0 & 0 & 0 & 0 \\
        0 & 0 & 0 & 0 & 0 & 0 & 0 & 0 \\
        0 & 0 & 0 & 0 & 0 & 0 & 0 & 0 \\
        0 & 0 & 0 & 0 & 0 & 0 & 0 & 0 \\
        0 & 0 & 0 & 0 & 0 & 0 & 0 & 0 \\
        0 & 1 & 0 & 1 & 0 & 1 & 0 & 1 \\
        1 & 0 & 1 & 0 & 1 & 0 & 1 & 0 \\
        1 & 1 & 1 & 1 & 1 & 1 & 1 & 1 \\
    \end{bmatrix}
$$
\begin{remark}
The action of this matrix is to encode the input value (as ``binary''), and send it to all \hthree heads.
E.g., $\vV_t^{(1)} = \vV_t^{(2)} = \vV_t^{(3)} = \vV_t^{(4)}$ for all $i$, and $\vV_t^{(i)} = [0, 0] \Leftrightarrow x_t = v_1$, $\vV_t^{(i)} = [0, 1] \Leftrightarrow x_t = v_2$, etc.
\end{remark}

We claim that for $x_{N+1} = k_i$, $\vO_{N+1}^{(i)}$ will be a multiple of the binary encoding of $f_x(k_i)$, and all the other heads of the output $\vO_{N+1}^{(j)}, 1 \leq j \leq 4, j \neq i$, will be zero.
Let the output projection $\vW_O$ be such that, with a non-linearity afterwards, it inverts the binary encoding to produce the embedding of the desired output $f_x(k_i)$.
We will assume such a projection exists, proof left to the reader.

\begin{proposition}\label{thm:h3_expressivity}
    The model described above solves the associative recall problem for the language $\Lambda$.
\end{proposition}

\begin{proof}
Proof sketch.
WLOG, let $x_{N+1} = k_i$.
Then $\vQ^{(i)} = [1, 1]$, but $\vQ^{(j)} = [0, 0]$ for $j \neq i$.
Thus, $\vO^{(j)} = [0, 0]$ for $j \neq i$ due to the multiplicative interaction.

Since $\vQ^{(i)} = [1, 1]$, $\vO^{(i)}$ is the output of the diag SSMs in the \hthree head corresponding to $k_i$ (recall that each head has two independent shift SSMs and two independent diag SSMs).
The output of the diag SSMs are the cumulative sum of all the inputs they have seen in the sequence.

For one of the diag SSMs to see a non-zero input, its preceding shift SSM must have a non-zero output.
The only times $t$ this can happen in the sequence are when $x_{t-1} = k_i$.
But then $x_t = f_x(k_i)$.
Thus, the input to the diag SSMs are precisely the binary encoding of $f_x(k_i)$.
Then the output $\vO^{(i)}$ is a multiple of the binary encoding of $f_x(k_i)$, $\vW_O$ decodes this output to the embedding of $f_x(k_i)$.
\end{proof}

\subsection{Attention Expressivity}
\label{sec:app_attention_expressivity}
We provide an informal sketch of a two-layer attention model that can solve the associative recall task, inspired by the construction of~\citep{olsson2022context}.
The first layer of the attention model outputs the embedding of the previous token in the sequence, and concatenates it with the current token in the sequence.
The second layer compares the current token to the previous token embeddings, and outputs the paired embedding when there is a match---which is exactly the key-value lookup.

The construction proceeds as follows:
\begin{itemize}
  \item In the first layer, let $Q_i$ be mapped to the positional embedding of token $x_{i-1}$ (e.g., $p_{i-1}$ if $p_i$ denotes the positional embedding of token $x_i$), and $K_i$ be mapped to the positional embedding of token $x_i$.
  \item The attention matrix $A$ is computed as $QK^T$, with a causal mask (i.e., $A_{i, j} = 0$ if $j > i$).
  \item Then, $softmax(A)$ approximates the shift matrix (see Section~\ref{sec:method}).
  \item Let $V_i$ be an encoding of token $x_i$, constrained to the first half of the hidden dimension.
  \item Then, for output $O = softmax(QK^T)V$, the first half of the vector $O_i$ is the encoding of token $x_{i-1}$.
  \item In the second layer, assume that you have a skip connection, that maps the encoding of the input token $x_i$ to the second half of the vector $O_i$.
  \item Then, the input to the second layer encodes both $x_{i-1}$ and $x_i$.
  \item In the second layer, let $Q_i$ extract the encoding of $x_i$, and let $K_i$ extract the encoding of $x_{i-1}$.
  \item Apply a causal mask on $QK^T$. Then, the value of $softmax(QK^T)_{i, j}$ is large if $x_i = x_{j-1}$, and $i > j - 1$.
  \item Let $V_i$ extract the encoding of $x_i$.
  \item Then, output $O_i$ is the sum of values $x_j$ such as $x_{j-1} = x_i$. But then $O_i$ is exactly a lookup of the token that came after $x_i$ when it appeared previously in the sequence---which exactly solves associative recall.
\end{itemize}

We note that the above construction requires the ability for the positional encodings to select the previous token based on the dot product and softmax, and for token comparisons through the dot product and softmax.

\subsection{\hthree Complexity}
\label{sec:app_h3_complexity}

We prove~\cref{thm:h3_complexity}, which states that the \hthree layer takes
$O(d^2N + d N \log N)$ time and $O(dN)$ space for sequence length $N$ and hidden
dimension $d$.

\begin{proof}
  We first analyze the time complexity.
  Consider the matrix multiplies in \hthree, where the input $u \in \mathbb{R}^{N \times d}$ is
  multiplied by three weight matrices of size $d \times d$. These take time $O(d^2N)$.
  The output $\vO$ is also multiplied with an output projection weight matrix of
  size $d \times d$, also taking time $O(d^2N)$.
  Therefore the matrix multiplies take time $O(d^2N)$.

  Now consider the two SSMs in \hthree. The first SSM involves a convolution of
  $\vK \in \mathbb{R}^{N \times d}$ (in the $N$-dimension) with a kernel of size $N \times d$.
  This reduces to an FFT, a pointwise multiply, and an inverse FFT (in the
  $N$-dimension).
  This takes time $O(d N \log N)$.
  The second SSM involves $H$ convolutions, inputs of size $N \times d_h \times d_h$,
  along the $N$-dimension.
  This takes time:
  \begin{equation*}
    O(H d_h^2 N \log N) = O(d d_h N \log N) = O(d N \log N),
  \end{equation*}
  where we use the fact that $d_h = d / H$ and that $d_h = O(1)$.
  Therefore the two SSMs take total time $O(d N \log N)$.
  As a result, the \hthree layer takes time:
  \begin{equation*}
    O(d^2N + d N \log N).
  \end{equation*}

  Now we analyze the space complexity.
  The matrix multiplies all take space $O(dN)$.
  The FFTs, pointwise multiplies, and inverse FFTs of the two SSMs takes $O(dN)$
  space and $O(H d_h^2N) = O(dd_hN) = O(dN)$ space.
  Therefore the overall space complexity is $O(dN)$.

\end{proof}

\subsection{State Passing Correctness}
\label{sec:app_statepassing_correctness_proof}

We prove~\cref{thm:statepassing_correctness}.
We assume that the \textsc{BlockFFTConv} algorithm is correct, i.e., the output $y = $\textsc{BlockFFTConv}$(f, u)$ is equal to the output of an SSM with convolution kernel $f$ and input $u$.

\begin{proof}  
Proof by induction on $C$.

\paragraph{Base case:} $C = 1$.
WTS $y = [y^{(1)}]$, $\vM_{xx}x_{N'}^{(0)} + \vM_{ux}u^{(1)} = x_{N}$.

In this case, note that $N = N'$.
Then $y^{(1)} = \vM_{xy}x_{N'}^{(0)} + $\textsc{BlockFFTConv}$(f, u_1) = $\textsc{BlockFFTConv}$(f, u_1)$.
But $u = u_1$, so $y = y^{(1)} = [y^{(1)}]$.

Additionally, by the recursive definition of a state space,
\begin{align*}
 x_{N} &= \vA^{N-1}x_0 + \sum_{i=1}^N \vA^{N-i}\vB u_i \\
 &= \vA^{N'-1}x_0 + \sum_{i=1}^{N'} \vA^{N'-i}\vB u^{(1)}_i \\
 &= \vM_{xy}x_{N'}^{(0)} + [\vA^{N'-1}\vB, \vA^{N'-2}\vB, \dots, \vB]u^{(1)}. \\
 &= \vM_{xy}x_{N'}^{(0)} + \vM_{ux}u^{(1)}.
\end{align*}

\paragraph{Inductive step:} $C > 1$.
Assume that $[y^{(1)}, \dots, y^{(C-1)}] = y[:N'(C-1)]$, and $x_{N'}^{(C-1)} = x_{(C-1)N'}$.
WTS that $y^{(C)} = y[N'(C-1):N'C]$, and $\vM_{xx}x_{N'}^{(C-1)} + \vM_{ux}u^{(C)} = x_{N}$.
Let $t$ denote $N'(C-1)$.

For $i > (C-1)N'$, we have:
\begin{align*}
y_i &= \vC\vA^{i-t}\vB x_{t} + (f \ast [u_{t}, u_{t+1}, \dots, u_{t+N'-1}])_{i-t} + \vD u_i \\
&= \vC\vA^{i-t}\vB x_{t} + (f \ast u^{(C)})_{i-t} + \vD u_i \\
&= \vC\vA^{i-t}\vB x_{t} + \textsc{BlockFFTConv}(f, u^{(C)})_{i-N'} \\
&= (\vM_{xy}x_t + \textsc{BlockFFTConv}(f, u^{(C)}))_{i-N'} \\
&= (\vM_{xy}x_{N'}^{(C-1)} + \textsc{BlockFFTConv}(f, u^{(C)}))_{i-N'} \\
&= y^{(C)}_{i-N'}.
\end{align*}
Thus, $y^{(C)} = y[N'(C-1):N'C]$.

Similarly,
\begin{align*}
x_N &= \vA^{N'-1}x_{(C-1)N'} + \sum_{i=1}^{N'} \vA^{N'-i}\vB u_{i+t} \\
&= \vA^{N'-1}x_{N'}^{(C-1)} + \sum_{i=1}^{N'} \vA^{N'-i}\vB u^{(C)}_i \\
&= \vM_{xx}x_{N'}^{(C-1)} + [\vA^{N'-1}\vB, \vA^{N'-2}\vB, \dots, \vB]u^{(C)} \\
&= \vM_{xx}x_{N'}^{(C-1)} + \vM_{ux}u^{(C)}.
\end{align*}
\end{proof}

\section{Experimental Details \label{sec:app_exp_details}}

\subsection{Synthetics}
Our synthetic tasks, inspired by~\citep{olsson2022context}, are designed to mimic the in-context learning capability of large language models---the ability to learn from examples in the input sequence, and use information from the input to generate the right answer for the output.
For example, the induction head task requires memorizing the token that appears after the special $\vdash$ token in the input sequence, and the associative recall task requires learning the mapping from keys to tokens from the input sequence.

We evaluate synthetics by training two-layer versions of our GPT models, with different modules replacing attention.
We train models with inner dimension 32, and MLP dimension 128.
For all the synthetics, we use a learning rate of 5e-4 and a weight decay of 0.1.
We sample 5000 training examples and 500 test examples from the same distribution, and we train for 200 epochs.
Again, we use embedding dropout of 0.1 and residual dropout of 0.0.

\subsection{Model Architecture}
For our 125M models, we use 12 layers, with hidden dimension 1024, and MLP dimension 4096.
For our 355M models, we use 24 layers, with the same hidden dimension and MLP dimension.
The 1.3B models have 24 layers, with hidden dimension 2048, and MLP dimension 8192.
The 2.7B models have 32 layers, hidden dimension 2560, and MLP dimension 10240.
The hybrid models have 12, 16, 16, and 20 heads for the 125M, 355M, 1.3B, and 2.7B models, respectively.
The 125M hybrid model has an attention layers at layers 1 and 7, the 355M and 1.3B hybrid models have attention layers at layers 1 and 13, and the 2.7B hybrid model has attention layers at layers 10 and 21.
For both our hybrid models and our \hthree models, we use SSM state size 64.
Our hybrid model uses head dimension 1 for \hthree, while our pure \hthree model uses head dimension 8.
We run models with mixed-precision training, with bf16 for the MLP's and attention.
When training language models, we use fp32 for the FFTConv.

\subsection{OpenWebText Training}
For the 125M models trained on OpenWebText, we follow the training recipe of the Megatron-LM repo.

We use an effective batch size of 512, and use gradient accumulation to fit into
available GPU memory.
We use the AdamW optimizer, with learning rate 6e-4 for GPT-2 small and 1.5e-4
for GPT-2 medium, and weight decay of 0.1.
All models are trained with the same hyperparameters for 100K steps.
We run all implementations with mixed-precision training (PyTorch AMP).
We train models with sequence length 1024.

We use the Openwebtext dataset, with the GPT-2 BPE tokenizer. We randomly select
0.5\% of the dataset as the validation set, with the rest being used as training
set.
This random selection of validation set is done once, and all models are evaluated
on the same validation set.

\subsection{The Pile Training}
For the 125M and 355M models trained on the Pile, we follow the training recipe of GPT-3.
We use batch size 256, with sequence length 2048.
We train our models for 800K steps.
We use residual dropout 0.0 and embedding dropout 0.1.
We use the AdamW optimizer, with learning rate 6e-4 for the 125M model and 3e-4 for the 355M model, and a weight decay of 0.1.
We use a cosine schedule with 8000 steps for linear warmup, and decay the
learning rate to 10\% by 300B tokens, then continue training at 10\% learning
rate for another 100B tokens.
We suspect that there exist better hyperparameters for \hthree language models, but we did not have the resources to tune them.

For the 1.3B models, we double the batch size to 512 (with sequence length
2048), again following the training recipe of GPT-3. The number of training
steps are halved so that we train on the same number of tokens.

For the Pile dataset, we again use the GPT-2 BPE tokenizer, similar to GPT-3 and GPT-Neo.

\subsection{SuperGLUE}
We follow the prompts used in the GPT-3 paper~\citep{brown2020language}.
For rank classification on the binary classification tasks, we use yes/no for WSC, WIC, MultiRC, and BoolQ, and we use true/false for RTE.
For CB, we use true/false/neither as the three choices.
For COPA and ReCoRD, we use the continuations provided by the task.

\subsection{Hardware}
All models were trained on either a single 16xA100-40GB node or a cluster of 8xA100-80GB nodes.

\section{Additional Experiments}
\label{sec:app_additional_experiments}

\subsection{LRA Accuracy}
We evaluate the accuracy of \hthree on LRA.
We compare accuracy to S4D~\citep{gu2022parameterization}, since \hthree uses an S4D kernel as a component in its layer.
We use the same hyperparameters as S4D, and make the layer bidirectional by making two copies and running them in opposite directions.

\begin{table}[ht]
    \small
    \centering
    \caption{\label{table:lra_acc} LRA performance of \hthree compared to S4D. }
    {
        \begin{tabular}{@{}|c|cccccc|c|@{}}
            \hline
        Model & ListOps & Text & Retrieval & Image & Pathfinder & Path-X & Avg  \\ 
        \hline
        S4D~\citep{gu2022parameterization} & \textbf{58.3} & 87.3 & 90.7 & \textbf{87.5} & \textbf{93.6} & \textbf{92.3} & \textbf{85.0} \\
        \hthree & 57.5 & \textbf{88.2} & \textbf{91.0} & 87.3 & 93.0 & 91.8 & 84.8  \\ \hline
        \end{tabular}
    }
\end{table}

Table~\ref{table:lra_acc} shows that \hthree performs well on the LRA benchmark, even thought it was designed for autoregressive language modeling.
\hthree outperforms S4D on two of the LRA tasks, and comes within 1 point on the others.

\subsection{WikiText103}
We train 125M-sized models on WikiText103~\citep{merity2016pointer} and compare their test PPL to transformers, as well as other variants of efficient or long-range attention.
We use the same hyperparameters and setup as training on OpenWebText.
We also provide results from Transformer-XL and Perceiver AR for context, though the results may not be directly comparable due to differences in model size and tokenizer.

\begin{table}[h]
\caption{\label{table:wikitext103} Test PPL on WikiText103.}
\centering
\small
\begin{tabular}{|c|c|}
\hline
Models &  PPL \\
\hline
Transformer (125M) & 18.6  \\
Hybrid \hthree (125M) & 18.5 \\
Performer (125M)~\citep{choromanski2020rethinking} & 26.8 \\
Reformer (125M)~\citep{kitaev2020reformer} & 26.0 \\
Linear Attention (125M)~\citep{katharopoulos2020transformers} & 25.6 \\ \hline
Perceiver AR (358M)~\citep{hawthorne2022general} & 18.4 \\
Transformer-XL (285M)~\citep{dai2019transformer} & 18.4 \\ \hline
\end{tabular}
\end{table}

Table~\ref{table:wikitext103} shows that the Hybrid \hthree model is competitive with Transformers of the same size, as well as larger models such as the 358M Perceiver AR and 285M Transformer-XL.
The hybrid \hthree model also significantly outperforms transformers with performer, reformer, and linear attention.

We note that the Transformer-XL and Perceiver AR PPl numbers are from the original papers, and may not be directly comparable to our results.
In particular, they use a tokenizer with a different vocab size, which means that the PPLs are not directly comparable.
In addition, the larger vocab size necessitates a change in the model (adaptive softmax) that may affect performance.
The top five numbers in Table~\ref{table:wikitext103} are trained with the same setup and are directly comparable to each other.

\subsection{PG-19}
We evaluate models trained on the PG-19 dataset~\citep{rae2019compressive}, a natural language dataset comprised of texts from books.
We compare the performance of Hybrid \hthree compared against Transformers and linear attention.
We use the same setup as evaluating on OpenWebText.

\begin{table}[h]
\caption{\label{table:pg19} Test PPL on PG-19.}
\centering
\small
\begin{tabular}{|c|c|}
\hline
Models &  PPL \\
\hline
Transformer (125M) & 17.0  \\
Hybrid \hthree (125M) & 16.2 \\
Linear Attention (125M)~\citep{katharopoulos2020transformers} & 19.1 \\ \hline
\end{tabular}
\end{table}

Table~\ref{table:pg19} shows that Hybrid \hthree outperforms transformers and linear attention.

\subsection{Length Extrapolation}
One property of SSMs is that they can naturally extrapolate to sequence lengths longer than those seen during training.
We use the synthetic associative recall task to demonstrate that \hthree maintains this capability.
To do so, we train a two-layer \hthree model on sequences of length 20 drawn from the associative recall synthetic language.
Then, we evaluate accuracy of the last token prediction on sequences of length 20 and 40.

\begin{table}[h]
\caption{\label{table:length_extrapolation} Accuracy of an \hthree model trained for associative recall on sequences of length 20, evaluated on sequences of length 20 and 40.}
\centering
\small
\begin{tabular}{|c|cc|}
\hline
Models & Acc, seqlen 20 & Acc, seqlen 40 \\
\hline
\hthree & 99.8 & 98.4 \\ \hline
\end{tabular}
\end{table}

Table~\ref{table:length_extrapolation} shows that \hthree maintains accuracy on sequences of length 40, which is twice the length of the training sequences.

\subsection{Scaling in Number of Tokens}
We evaluate how well a Hybrid \hthree model scales with the number of tokens seen during training, compared to a Transformer.
For these experiments, we train a 125M Hybrid \hthree model and a 125M Transformer on the Pile for 5B, 10B, and 15B tokens.
We use a learning rate of 6e-4, adjusting the warmup to be 1\% of the total training time, and adjusting the decay rate to decay the learning rate to 6e-5 by the end of training.

\begin{table}[h]
\caption{\label{table:scaling_tokens} Test PPL on the Pile for models trained with fewer tokens.}
\centering
\small
\begin{tabular}{|c|cc|}
\hline
Train Tokens & Hybrid \hthree (125M) & Transformer (125M) \\
\hline
5B & 11.8 & 12.7 \\
10B & 10.7 & 11.3 \\
15B & 10.2 & 10.7 \\
\hline
\end{tabular}
\end{table}

Table~\ref{table:scaling_tokens} shows the results.
Both the Hybrid \hthree model and Transformer model improve as the number of training tokens increases.

\subsection{\hthree Language Model}
\begin{table}[ht]
    \small
    \centering
    \caption{\label{table:superglue_zeroshot_all} Zero-shot performance on SuperGLUE with rank classification. Best results for each model size in bold. }
    {
        \begin{tabular}{@{}|c|cccccccc|c|@{}}
            \hline
        Model & WSC & WIC & RTE & CB & MultiRC & ReCoRD & BoolQ & COPA & Average \\ 
        \hline
        OPT-125M & \underline{39.4} & \underline{52.0} & 48.7 & 37.4 & \underline{58.9} & \underline{44.9} & \underline{59.6} & \underline{60.0} & 50.1 \\
        GPT-Neo-125M & 36.5 & \textbf{53.6} & \underline{53.1} & \underline{41.1} & \textbf{59.9} & 39.6 & \textbf{62.2} & \underline{60.0} & \underline{50.8} \\
        \textbf{\hthree-125M} & \textbf{61.5} & 50.0 & \underline{53.1} & \underline{41.1} & 4.6 & 15.8 & 46.4 & 51.0 & 40.4 \\
        \textbf{Hybrid \hthree-125M} & \underline{39.4} & 51.4 & \textbf{59.2} & \textbf{48.2} & 51.4 & \textbf{55.0} & \underline{59.6} & \textbf{67.0} & \textbf{53.9} \\ \hline 
        \end{tabular}
    }
\end{table}
\begin{table}[h]
    \small
    \centering
    \caption{\label{table:superglue_fewshot_all} 3-shot performance on SuperGLUE with rank classification. Best results for each size in bold, second best underline. }
    {
        \begin{tabular}{@{}|c|cccccccc|c|@{}}
            \hline
        Model & WSC & WIC & RTE & CB & MultiRC & ReCoRD & BoolQ & COPA & Average \\ 
        \hline
        OPT-125M & 36.5 & \textbf{50.2} & 47.3 & 44.6 & \textbf{57.9} & \underline{44.9} & 41.9 & \underline{60.0} & \underline{47.9} \\
        GPT-Neo-125M & 38.5 & \underline{50.0} & \underline{53.1} & 17.9 & \underline{56.3} & 39.6 & \textbf{62.1} & \underline{60.0} & 47.2 \\
        \textbf{\hthree-125M} & \textbf{63.5} & \underline{50.0} & 52.3 & \underline{48.2} & 32.6 & 15.8 & 37.8 & 51.0 & 43.9 \\ 
        \textbf{Hybrid \hthree-125M} & \underline{43.3} & 49.1 & \textbf{58.1} & \textbf{51.8} & 48.9 & \textbf{55.0} & \underline{56.1} & \textbf{67.0} & \textbf{53.7} \\ \hline 
        \end{tabular}
    }
\end{table}

We report the results of a pure \hthree language model on NLP evaluations.
We train a 125M model on the Pile for 400B tokens.
Tables~\ref{table:superglue_zeroshot_all} and~\ref{table:superglue_fewshot_all} show zero-shot and few-shot performance on SuperGLUE, respectively.

\subsection{Generation Performance\label{sec:app_generation}}
\begin{table}[h]
    \scriptsize
    \centering
    \caption{\label{table:superglue_zeroshot} Zero-shot performance on SuperGLUE. Best results for each size in bold, second best underline. }
    {
        \begin{tabular}{@{}|c|cccccccc|c|@{}}
            \hline
        Model & WSC & WIC & RTE & CB & MultiRC & ReCoRD & BoolQ & COPA & Average \\ 
        \hline
        OPT-125M & \textbf{36.5} & \textbf{48.4} & \textbf{49.8} & \textbf{8.9} & \textbf{39.1} & \underline{44.9} & 45.9 & \underline{60.0} & \textbf{41.7} \\
        GPT-Neo-125M & \underline{27.9} & \underline{11.3} & 45.8 & \textbf{8.9} & \underline{19.1} & 39.6 & \textbf{56.4} & \underline{60.0} & \underline{33.6} \\
        \textbf{Hybrid \hthree-125M} & 0.0 & 0.0 & \underline{47.3} & \textbf{8.9} & 4.4 & \textbf{55.0} & \underline{47.6} & \textbf{67.0} & 28.8 \\ \hline 
        GPT-2 medium (355M) & \textbf{50.0} & \textbf{50.2} & 16.2 & \textbf{21.4} & 10.5 & \underline{53.3} & 38.4 & \underline{65.0} & \underline{38.1} \\
        OPT-350M & \underline{41.3} & \underline{34.8} & \textbf{49.5} & \underline{16.1} & \textbf{23.6} & 51.4 & \underline{39.7} & 60.0 & \textbf{39.6} \\
        \textbf{Hybrid \hthree-355M} & 22.1 & 21.5 & \underline{47.3} & 8.9 & \underline{17.1} & \textbf{62.3} & \textbf{44.4} & \textbf{69.0} & 36.6 \\ \hline
        \end{tabular}
    }
\end{table}
\begin{table}[h]
    \scriptsize
    \centering
    \caption{\label{table:superglue_fewshot} 3-shot performance on SuperGLUE with generation. Best results for each size in bold, second best underline. }
    {
        \begin{tabular}{@{}|c|cccccccc|c|@{}}
            \hline
        Model & WSC & WIC & RTE & CB & MultiRC & ReCoRD & BoolQ & COPA & Average \\ 
        \hline
        OPT-125M & 36.5 & \underline{49.1} & 47.3 & 33.9 & \underline{35.5} & \underline{44.8} & 38.5 & 60.0 & 43.2 \\
        GPT-Neo-125M & \underline{38.5} & \textbf{50.0} & \underline{53.1} & \textbf{42.9} & 22.5 & 39.7 & \textbf{61.2} & \textbf{68.0} & \underline{47.0} \\
        \textbf{\hthree-125M} & 0.0 & 0.0 & 47.3 & 8.9 & 0.0 & 15.4 & 37.8 & 53.0 & 20.3 \\ 
        \textbf{Hybrid \hthree-125M} & \textbf{43.3} & \underline{49.1} & \textbf{58.1} & \underline{41.1} & \textbf{40.3} & \textbf{55.2} & \underline{49.5} & \underline{67.0} & \textbf{50.5} \\ \hline 
        GPT-2 medium (355M) & 36.5 & \textbf{50.5} & 47.3 & 28.6 & 35.3 & \underline{53.1} & \underline{37.8} & \underline{63.0} & 44.0 \\
        OPT-350M & \underline{37.5} & \underline{50.0} & \underline{46.2} & \textbf{41.1} & \underline{40.6} & 51.3 & 39.4 & 59.0 & \underline{45.6} \\
        \textbf{Hybrid \hthree-355M} & \textbf{42.3} & 47.5 & \textbf{50.5} & \underline{37.5} & \textbf{57.5} & \textbf{61.4} & \textbf{45.4} & \textbf{73.0} & \textbf{51.9} \\ \hline
        \end{tabular}
    }
\end{table}
We report results on SuperGLUE for generation.
Instead of taking rank classification, we instead let the model generate a response, and we search for the gold label (i.e., ``yes'' or ``no'' for the yes/no questions) in the output.
Tables~\ref{table:superglue_zeroshot} and~\ref{table:superglue_fewshot} report the results.
The trends for few-shot learning match with the logit results, but the hybrid and \hthree models perform very poorly in zero-shot performance on some tasks.
In these cases, the models tend to generate long text responses that are not relevant to the answer.
The few-shot learning examples help the models generate answers in a parsable format.

\subsection{Non-Text Sequence Modeling}
\label{subsec:non_text_sequence_modeling}

We show that \hthree outperforms Transformers on two non-text sequence modeling tasks: raw speech classification and seizure classification over raw EEG signals.
\hthree sets state-of-the-art performance on seizure classification and matches S4 on speech classification---which suggests that \hthree, or one of its hybrids, may be a strong candidate for a multimodal foundation model.
Appendix~\ref{sec:app_exp_details} gives experimental details, and Appendix~\ref{sec:app_additional_experiments} gives an additional experiment on brain fMRI data.

\paragraph{Seizure Classification from EEG}
Seizures, which are characterized by uncontrolled brain activity, are some of the most common neurological disorders~\citep{fisher2014ilae}. Chronic seizures, or epilepsy, cause a range of psychiatric and psycho-social disorders and impact the lives of roughly one percent of the global population~\citep{kerr2012impact}. The first step to treating epilepsy is manual analysis of scalp EEG by board-certified neurologists. However, the vast amount of EEG data produced by each patient (which can be up to days of data) makes manual EEG analysis a costly and time-consuming process.  

To mitigate the costs associated with EEG monitoring, recent deep learning techniques have began to show promise in flagging abnormal EEG segments for potential seizure events~\citep{siddiqui2020review}. A challenge with classifying EEG data is the trade-off between increasing input sequence length, where more context has been shown to improve seizure classification performance \cite{saab2020weak}, with the increased difficulty of training deep learning models on long sequences (e.g., an EEG signal sampled at $200$Hz produces $12{,}000$ time steps per minute). As a result, many techniques involve domain-specialized models and pre-processing steps, such as FFT transforms and graphical representations \cite{tang2021self}.

We use the largest publicly available EEG corpus, TUSZ v1.5.2~\citep{shah2018temple}, which includes $5{,}612$ EEG signals from 636 patients, with $3{,}050$ annotated seizures.
Signals are segmented into 60-second clips, and split into train/val/test by patient.
The train set contains 39765 clips, the val set contains 4351 clips, and the test set contains 10001 clips.

\begin{table}[h]
    \small
    \centering
    \caption{\label{table:eeg} Performance (AUROC) on 60s seizure classification from raw EEG (sequence length 12000).}
    {
        \begin{tabular}{@{}|cccccc|@{}}
        \hline
        H3 & Transformer  & Dense-CNN & CNN-LSTM & LSTM & 1D-CNN  \\ 
        \hline
        \textbf{83.2} & x  & 78.0 & 68.6 & 69.3 & 69.7 \\
        \hline
        \end{tabular}
    }
\end{table}

We evaluate binary seizure classification of $60$-sec EEG clips, sampled at $200$Hz, with 19 electrodes: $ x \in R^{12{,}000 \times 19}$ and $y \in \{0,1\}$ on the TUSZ v1.5.2~\citep{shah2018temple} corpus.
Transformers cannot process the long sequence length of EEG signals without running out of GPU memory, whereas \hthree can---and sets state-of-the-art performance.

\paragraph{Raw Speech Classification}
The SC10 speech commands task~\citep{warden2018speech} contains raw audio signals one second in length, sampled at 16kHz.
Similarly to EEG signals, Transformers cannot process the long sequence length.
Table~\ref{table:speech} shows that \hthree comes within half a point of S4, the state-of-the-art method.

\begin{table}[h]
    \small
    \centering
    \caption{\label{table:speech} SC 10-class classification on raw audio (sequence length 16000).}
    {
        \begin{tabular}{@{}|cccccc|@{}}
            \hline
        H3 & S4 & WaveGan-D & Transformer & Performer & CKConv  \\ 
        \hline
        97.04 & \textbf{97.50} & 96.25 & x & 30.77 & 71.66 \\ \hline
        \end{tabular}
    }
\end{table}

\paragraph{Functional Magnetic Resonance Imaging Data}

Functional Magnetic Resonance Imaging (fMRI) data are typically represented in four dimensions, indicating the measured blood-oxygen-level-dependent (BOLD) signal in temporal sequences $S = \{V_1, ..., V_t\}$ of 3-dimensional volumes $V \in \mathbb{R}^{x \times y \times z}$, each indicating the BOLD signal for all spatial locations of the brain (as defined by three spatial dimensions $x$, $y$, and $z$). A key challenge for the analysis of fMRI data is the high dimensionality and low sample size of its datasets, which typically contain many hundred thousand dimensions (i.e., voxels) for each of several hundred volumes $V$ in each of tens to hundreds of sequences $S$. In this setting, where the number of features exceed the number of samples, standard machine learning approaches are prone to overfitting.

In spite of the low sample size of individual datasets, neuroimaging research can be considered as recently entering a big data era because researchers more frequently share their collected datasets publicly~\citep{markiewicz_2021_openneuro}. The availability of these data open up the opportunity for pre-training in neuroimaging at scale, as recently demonstrated by~\citep{thomas_fmri_2022}, enabling models to utilize the knowledge that can be learned from public functional neuroimaging data for the analysis of individual datasets. Specifically,~\citep{thomas_fmri_2022} evaluate the performance of several self-supervised learning frameworks for functional neuroimaging data by first pre-training models on a broad fMRI dataset spanning $11,980$ fMRI runs from $1,726$ individuals across $34$ datasets and subsequently adapting the pre-trained models to two downstream mental state decoding datasets (namely, the HCP~\citep{van_2013_wu} and MDTB~\citep{king_2019_functional} datasets). In mental state decoding, predictive models are tasked with identifying (i.e., decoding) some set of mental states (e.g., answering questions about a prose story or math problem) from measured brain activity. The authors find that a GPT-based model, pre-trained in a causal learning framework, performs best in decoding the $20$ (HCP) and $26$ (MDTB) mental states of the two downstream datasets.

To evaluate the performance of H3 on fMRI data, we replicate this analysis, using the up- and downstream fMRI datasets that were published by~\citep{thomas_fmri_2022}, treating H3 as a drop-in replacement for the GPT model. To alleviate the high dimensionality challenge of fMRI data, and due to the generally high spatial correlation of brain activity, the original authors have reduced the volumetric time series $S$ to a set $\Theta \in {\theta_1, ..., \theta_n}$ of $n=1,024$ functionally-independent brain networks $\theta$  (as defined by the dictionaries of functional modes (DiFuMo) brain atlas~\citep{dadi_2020_fine}), each describing the BOLD signal for some subset of voxels $v_{x,y,z} \in V$, such that the resulting sequences $X \in \mathbb{R}^{t \times n}$ describe the activity pattern of each brain network $n$ for time points $t$. 

In line with~\citep{thomas_fmri_2022}, we first pre-train models $f(\cdot)$ to predict the distribution of brain activity for the next time point $j$ of an input sequence $X$, using a mean absolute error ($L_{rec}$) training objective given the model's output $\hat{X} \in \mathbb{R}^{t \times n}$: $L_{rec} = \frac{1}{n} \sum_{i=1}^{n} |X_{j,i} - \hat{X}_{j,i}|$; $\hat{X}_{t,n} = b_n + \sum_n f(E^{X})_{t,e} w_{e,n}$; $E^{X}_{t,e} = E^{TR} + E^{pos} + b_e + \sum_n X_{t,n} w_{n,e}$. Here, $E^{TR} \in \mathbb{R}^{e}$ and $E^{pos} \in \mathbb{R}^{e}$ represent learnable embeddings for each possible time point and position of an input sequence (for details, see~\citep{thomas_fmri_2022}). As the sampling frequency of fMRI is variable, the same position of an input sequence can correspond to different time points. Note that $f(\cdot)$ processes the input in a lower-dimensional embedding representation $E^{X} \in \mathbb{R}^{t \times e}$ (with $e=768$ dimensions).

We evaluate two model architectures for $f(\cdot)$, namely, the GPT variant used in ~\citep{thomas_fmri_2022}, with $4$ hidden layers and $12$ attention heads, and a corresponding H3 variant with $4$ hidden layers (with $H=64$ and $m=1$; see section \ref{sec:method}). For both models, the sequence of hidden-states outputs of the last model layer are used to determine $\hat{X}$.

Just as~\citep{thomas_fmri_2022}, we randomly divide the upstream data into distinct training and validation datasets by randomly designating $5\%$ of the fMRI runs of each of the $34$ upstream datasets as validation data (at a minimum of $2$ runs per dataset) and using the rest of the runs for training. During upstream learning, we then randomly sample sequences of $10$ to $100$ time points from the fMRI runs and train models with the ADAM optimizer (with $\beta_1=0.9$, $\beta_2=0.999$, and $\epsilon=1e^{-8}$ ) for $5,000$ steps at a mini-batch size of $512$, and a learning rate of $5e^{-4}$. We apply a linear learning rate decay schedule (with a warm-up phase of 1\% of the total number of training steps), gradient norm clipping at $1.0$, and $L2$-regularisation (weighted by $0.1$). We also apply dropout at a rate of $0.1$ for the GPT-based model (based on~\citep{thomas_fmri_2022}) and evaluate three dropout rates for H3: $0.1$, $0.2$, and $0.3$. 

\begin{figure}[h]
    \centering
    \includegraphics[width=0.65\textwidth]{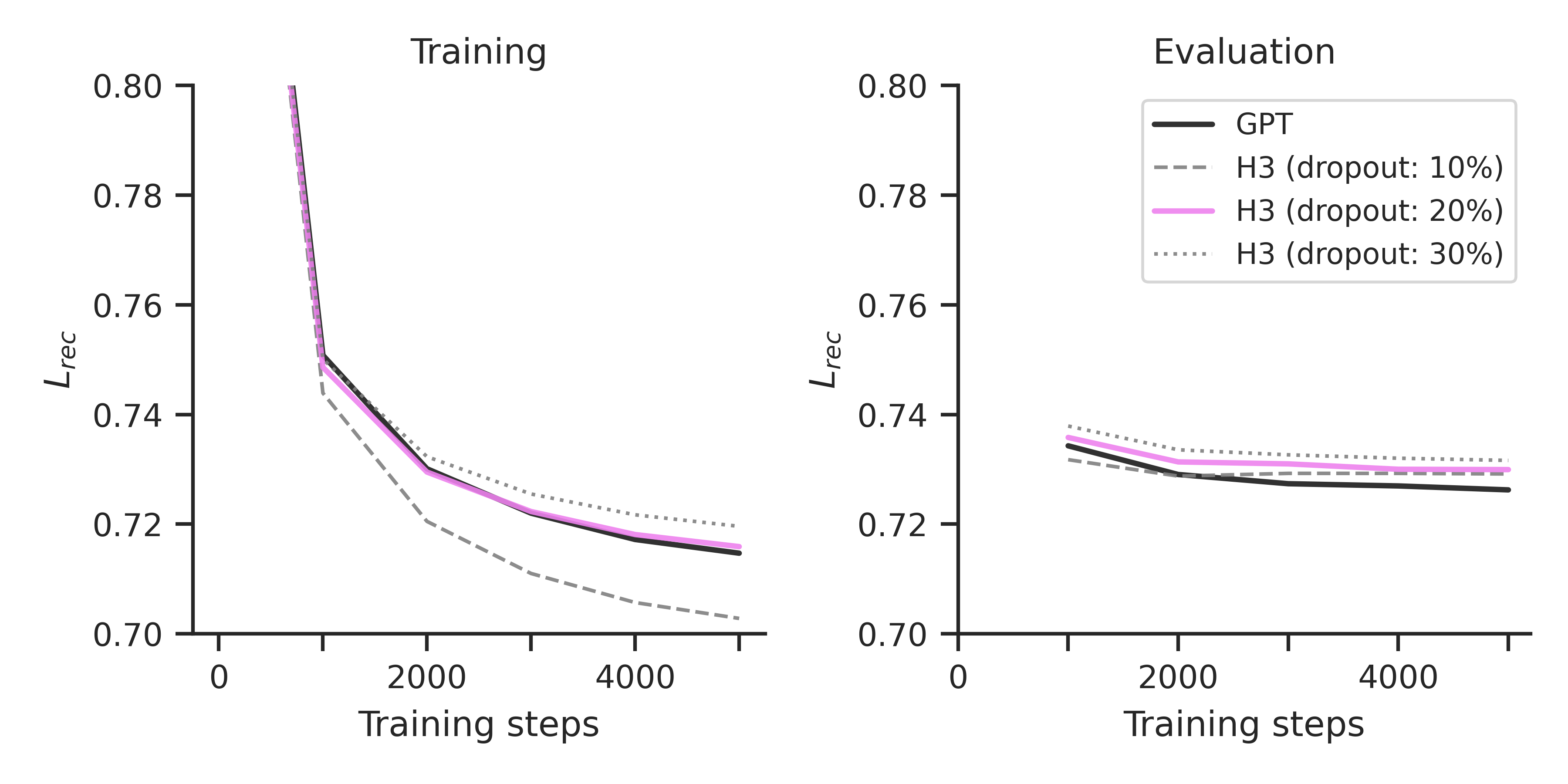}
    \caption{\label{fig:fmri-upstream} Upstream mean absolute error ($L_{rec}$) in training and evaluation datasets over the course of model training.}
\end{figure}

We find that the H3 variant trained with $0.2$ dropout performs on par with the GPT model, in terms of mean absolute error (Fig.\ \ref{fig:fmri-upstream}), and therefore continue all further analyses with this model variant. We also find that both models exhibit almost identify $L_{rec}$ error distributions throughout the brain, with relatively higher errors in the posterior parietal, occipital, and cingulate cortices as well parts of the limbic system (Fig.\ \ref{fig:fmri-upstream-brain}).

\begin{figure}[h]
    \centering
    \includegraphics[width=0.5\textwidth]{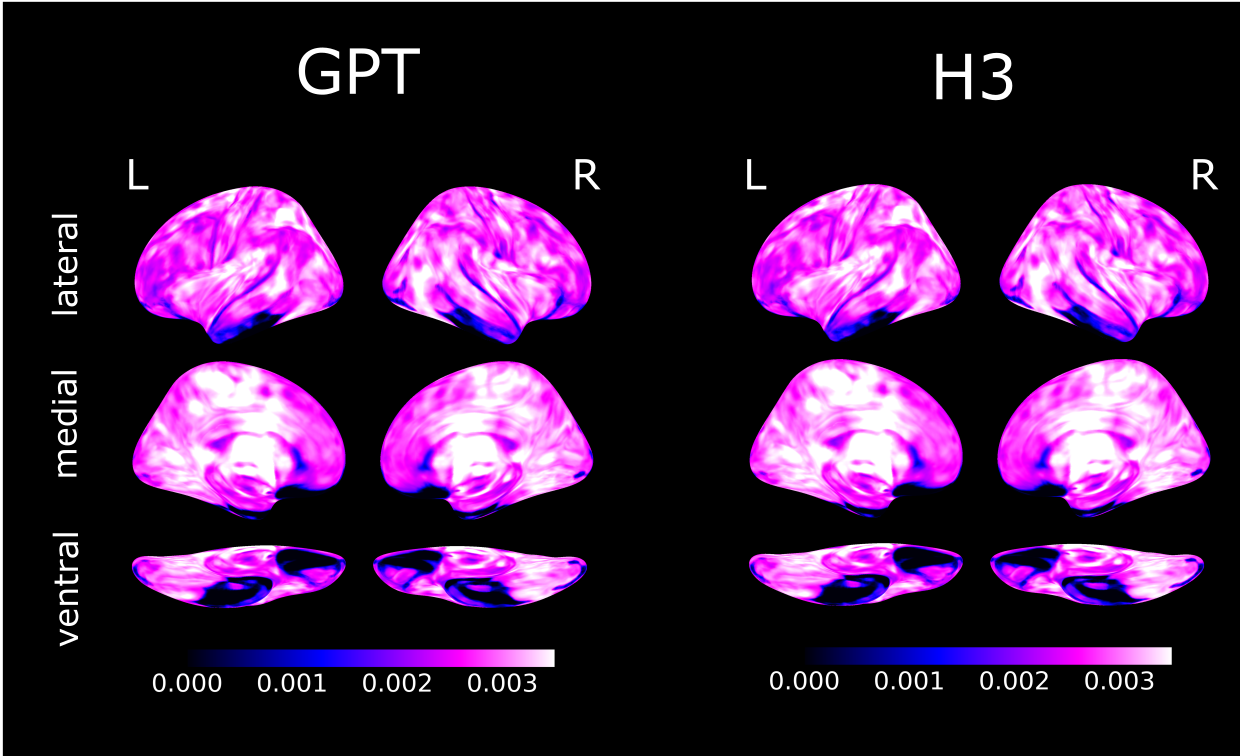}
    \caption{\label{fig:fmri-upstream-brain} Mean absolute error ($L_{rec}$) of the final pre-trained models for each voxel of the brain projected onto the inflated cortical surface of the FsAverage template~\citep{fischl_2012_freesurfer}.}
\end{figure}

To adapt the pre-trained models to mental state decoding, we add a learnable classification embedding $E^{cls} \in \mathbb{R}^{n}$ to the end of input sequences $X$ and forward the model's prediction $f(E^{X})$ to a decoding head $p(\cdot)$, composed of a dense hidden layer with $e$ model units (one for each embedding dimension, with $tanh$ activation) as well as a $softmax$ output layer with one model unit $i$ for each considered mental state in the data. Accordingly, we adapt models by optimizing a standard cross entropy loss objective: $L_{cls} = - \sum_i y_i \log \ {p(f(E^X))_i}$, where $y_i$ indicates a binary variable that is $1$ if $i$ is the correct mental state and $0$ otherwise.

During downstream adaptation, we begin training with the respective pre-trained model parameters and then allow all parameters to change freely. Similar to~\citep{thomas_fmri_2022}, we randomly split each of the two downstream datasets into distinct training and test datasets, each comprising $40$ (HCP) or $10$ (MDTB) distinct individuals. We adapt models for $750$ steps at a mini-batch size of $256$ and a learning rate of $5e^{-5}$ (otherwise using the same learning parameters as for upstream training). Importantly, we repeat each downstream training run $20$ times using different random seeds, leading to different random splits of the data and variability in other non-deterministic training factors (such as random initialization and data shuffling).

As for the upstream data, the H3 and GPT-based models generally perform on par in their mental state decoding performances in the two left-out test datasets (Table \ref{table:fmri-downstream}).

\begin{table}[h]
    \small
    \centering
    \caption{\label{table:fmri-downstream} Downstream adaptation performance of models pre-trained on fMRI data, averaged over $20$ training runs with varying random seeds. F1-scores are macro-averaged.}
    {
        \begin{tabular}{@{}|c|c|cc|@{}}
        \hline
        Dataset & Model & Acc. ($\pm 95\% \text{CI}$) & F1 ($\pm 95\% \text{CI}$)\\ 
        \hline
        HCP & GPT & $88.44$ ($\pm 0.39$) & $87.24$ ($\pm 0.39$) \\
         & H3 & $88.75$ ($\pm 0.33$) & $87.16$ ($\pm 0.37$) \\ \hline
        MDTB & GPT & $89.47$ ($\pm 0.44$) & $88.74$ ($\pm 0.54$) \\
         & H3 & $88.25$ ($\pm 0.45$) & $85.76$ ($\pm 0.61$) \\ \hline
        \end{tabular}
    }
\end{table}

\end{document}